\documentclass{article}
\usepackage{manuscript,times}
\usepackage[T1]{fontenc}
\usepackage[utf8]{inputenc}
\usepackage{amsmath,amssymb,amsthm,mathtools,bm}
\usepackage{mathabx}
\usepackage{pifont}
\newcommand{\defeq}{\vcentcolon=}
\usepackage{booktabs,tabularx,multirow,graphicx}
\usepackage{float}
\usepackage{enumitem}
\floatstyle{ruled}
\newfloat{algorithm}{htbp}{loa}
\floatname{algorithm}{Algorithm}
\usepackage{tikz,pgfplots}
\pgfplotsset{compat=1.18}
\usepgfplotslibrary{groupplots}
\usepackage{hyperref,url,etoolbox}
\usepackage{refcount}
\usepackage{cleveref}
\hypersetup{pdftitle={Retimed Bellman Flows: Escaping the Impossible Triangle of Velocity Bootstrapping},pdfauthor={Boyang Xu, Shengzhe Chen, Hao Yan},colorlinks=true,linkcolor=blue!45!black,citecolor=blue!45!black,urlcolor=blue!45!black}
\newcommand{\E}{\mathbb E}
\newcommand{\R}{\mathbb R}

\newif\ifshownotes
\definecolor{notecolor}{RGB}{31,92,171}
\definecolor{todocolor}{RGB}{200,30,30}

\theoremstyle{plain}
\newtheorem{theorem}{Theorem}[section]
\newtheorem{proposition}[theorem]{Proposition}
\newtheorem{lemma}[theorem]{Lemma}
\newtheorem{corollary}[theorem]{Corollary}
\theoremstyle{definition}
\newtheorem{definition}[theorem]{Definition}

\theoremstyle{remark}

\crefname{assumption}{Assumption}{Assumptions}

\newcommand{\cmark}{\ding{51}}
\newcommand{\xmark}{\ding{55}}
\newcommand{\gt}{\tilde{\gamma}}                 
\newcommand{\vtg}{\bar{v}}                         
\newcommand{\vth}{v_{\theta}}                     
\newcommand{\psib}{\bar{\psi}}                    
\newcommand{\cN}{\mathcal{N}}

\newcommand{\Tpi}{\mathcal{T}^{\pi}}
\newcommand{\Law}{\operatorname{Law}}
\newcommand{\sg}{\operatorname{sg}}

\newcommand{\Xnoise}{\tilde{X}_0}                 

\newcommand{\indep}{\perp\!\!\!\perp}

\newcommand{\eqd}{\stackrel{\mathrm{D}}{=}}

\newcommand{\rebf}{ReBF}

\newcommand{\valueflows}{Value Flows}

\newcommand{\dS}{\textbf{(S)}}
\newcommand{\dA}{\textbf{(A)}}
\newcommand{\dD}{\textbf{(D)}}

\newcommand{\Var}{\operatorname{Var}}
\newcommand{\Cov}{\operatorname{Cov}}
\newcommand{\Id}{I}
\newcommand{\norm}[1]{\left\lVert #1\right\rVert}
\newcommand{\inner}[2]{\left\langle #1,#2\right\rangle}

\newcommand{\T}{\mathcal T}
\newcommand{\B}{\mathfrak B}
\newcommand{\F}{\mathcal F}
\newcommand{\HH}{\mathcal H}

\newcounter{fullstatement}
\newcommand{\fullstatementref}{}

\newtheorem{fullthm}[fullstatement]{Theorem}
\newtheorem{fullprop}[fullstatement]{Proposition}
\newenvironment{theorem*}[2][]{\renewcommand{\fullstatementref}{#2}\begin{fullthm}[#1]}{\end{fullthm}}
\newenvironment{proposition*}[2][]{\renewcommand{\fullstatementref}{#2}\begin{fullprop}[#1]}{\end{fullprop}}
\theoremstyle{definition}

\definecolor{teal}{HTML}{287C8E}
\definecolor{rust}{HTML}{BE593A}
\definecolor{purple}{HTML}{5A57A4}
\title{Retimed Bellman Flows: Escaping the Impossible Triangle of Velocity Bootstrapping}
\iclrfinalcopy
\author{Boyang Xu \\
Arizona State University \\
\And
Shengzhe Chen \\
Arizona State University \\
\And
Hao Yan \\
Arizona State University}
\begin{document}
\maketitle

\begin{abstract}
Flow critics learn return distributions by transporting Gaussian noise to Bellman endpoints via continuous velocity fields. While velocity bootstrapping stabilizes training by querying a successor teacher, existing methods face a structural dilemma: on straight paths, no residual-free same-time affine mapping can preserve Gaussian initial noise while maintaining an unbiased target. To overcome this limitation, we introduce Retimed Bellman Flows (ReBF). ReBF queries the teacher critic at a dynamically shifted earlier flow time, aligning intermediate student and teacher trajectories. By combining this retimed clock with fresh, decoupled noise generation, ReBF constructs a provably conditionally unbiased velocity target that preserves the Bellman fixed point and contracts under Wasserstein distances. Empirically, ReBF reduces $W_1$ distance to ground-truth return distributions by up to $7.7\times$ on synthetic MRPs and outperforms existing flow critics across 38 challenging OGBench and D4RL offline reinforcement learning tasks.
\end{abstract}

\section{Introduction}
Flow critics represent the law of a return, not only its mean, by a velocity field that transports Gaussian noise to returns, trained by flow matching~\citep{lipman2023,dong2026,xu2026}. The return can be a scalar, as in distributional reinforcement learning~\citep{bellemare2017}, or vector-valued, such as successor features (Appendix~\ref{app:env-fourrooms}). Unlike categorical critics such as C51, which fix a support~\citep{bellemare2017}, and quantile critics such as IQN, whose quantile function has no canonical multivariate analogue~\citep{dabney2018iqn}, a flow critic transports noise to returns of any dimension, and flow critics have shown strong performance against both on offline reinforcement learning benchmarks~\citep{dong2026,xu2026}.

Training a flow-based critic via temporal-difference (TD) learning, however, introduces a fundamental mismatch between flow generation and Bellman updates. While flow matching requires learning a continuous trajectory over intermediate flow times $t \in [0,1]$, the distributional Bellman equation strictly constrains only the marginal distribution at the endpoint $t=1$, leaving intermediate flow dynamics unconstrained, inducing severe source inconsistency and path ambiguity across flow steps. Furthermore, directly regressing velocity vectors onto stochastically sampled successor returns suffers from high sample variance and unstable convergence. To address these issues, motivated by bootstrapping in standard value learning \citep{sutton1988td}, recent works \citep{dong2026,xu2026} employ \emph{velocity bootstrapping}: regressing the student's velocity field onto a target constructed from the teacher's velocity at the successor state--action pair. Under this framework, the \emph{student} learns a generative path transporting Gaussian noise to the Bellman target return, while the \emph{teacher} defines a path transporting noise to the next-state return $X'$. Consequently, the bootstrapping target crucially depends on how points along the student's intermediate trajectory are aligned with those of the teacher.

\begin{table}[t]
\centering
\caption{\textbf{Comparison of flow-critic target designs} across three design properties: source consistency (\textbf{S}), affinity (\textbf{A}), and decoupled generation (\textbf{D}). On the straight path, no residual-free same-time affine lift with a time-only scale holds for every endpoint pair (Corollary~\ref{cor:impossible}); non-affine, state-dependent, and other-path queries lie outside that statement. \rebf{} retimes the query and generates the successor with fresh independence. With a coupling-matched teacher and exact population regression, the target is unbiased on canonical fields.}
\label{tab:desiderata}\label{tab:paths}\label{tab:triangle}
\vspace{2pt}
\footnotesize
\setlength{\tabcolsep}{4pt}
\begin{tabular}{@{}lccc@{}}
\toprule
\textbf{Method} & \textbf{Source consistency (S)} & \textbf{Affinity (A)} & \textbf{Decoupled generation (D)} \\
\midrule
Value Flows \citep{dong2026} & \xmark{} $R{+}\gamma\tilde X_0{\neq}X_0$ & \cmark{} $Z_t{=}R{+}\gamma X'_t$ & \xmark{} \\
PCBF \citep{xu2026} & \cmark{} at $X_0$ & \xmark{} contains $X_0$ & \xmark{} \\
\rebf{} (ours) & \cmark{} at $X_0$ & \cmark{} warped affine & \cmark{} \\
\bottomrule
\end{tabular}
\end{table}

Formulating this alignment thus calls for a precise matching mechanism between the intermediate flow steps. That match should respect the distributional Bellman equation throughout the flow, which asks for three design properties (Section~\ref{sec:desiderata}). (i) \textbf{Source consistency (S):} the path starts from standard Gaussian noise, the same source used at generation. (ii) \textbf{Affinity (A):} the current point is an affine image of the teacher query, so the two carry the same information given the transition. (iii) \textbf{Decoupled generation (D):} the successor return is generated with fresh noise, $\rho_1=0$. Affinity and decoupled generation are the road to unbiasedness: when the teacher is the exact velocity for that pairing, the target is conditionally unbiased. A learned teacher can still leave approximation error. As Table~\ref{tab:desiderata} records, neither prior method has both \dA{} and \dD{}. The obstruction is narrower than a three-way impossibility: on the straight path, source consistency and a same-time query cannot be combined in a residual-free affine lift whose scale depends only on time and discount (Corollary~\ref{cor:impossible}). This leads to the following question:
\begin{center}
\textbf{Can flow-based velocity bootstrapping be source-consistent, affine, and decoupled at the same time, and therefore unbiased?}
\end{center}

\emph{Retimed Bellman Flows} (\rebf{}) queries the teacher at an earlier flow time, chosen so that the contributions of source noise and successor return align after accounting for the reward and discount. This preserves the student's Gaussian starting point and Bellman endpoint while giving each intermediate student point a direct affine correspondence with a teacher query. Geometric alignment alone does not ensure an unbiased velocity target. \rebf{} also generates the successor return with noise independent of the student's source noise, matching the pairing under which the reused critic learns its conditional velocity. When that teacher is the coupling-matched velocity and the population regression is exact, the bootstrapped target has the correct conditional mean at each student point. On canonical fields, under fresh independence, it preserves the Bellman fixed point. A learned teacher can still leave approximation error. The time shift is determined by the discount and flow time, requiring no additional tuning parameter or teacher evaluation. We formalize these claims in Sections~\ref{sec:method} and~\ref{sec:theory}.

Our main contribution is Retimed Bellman Flows (\rebf{}), a framework for flow-based distributional value estimation: \textbf{(i)}~On the straight source path, we prove that no residual-free same-time affine lift with a time-only scale holds for every endpoint pair; non-affine, state-dependent, and other-path methods lie outside that statement; \textbf{(ii)}~\rebf{} queries the teacher at a dynamically shifted earlier flow time $\tau(t)$ and generates the successor under fresh independence ($\rho_1=0$ together with (A6)), a closed-form target with no added hyperparameter; \textbf{(iii)}~we prove that this retimed shift is the unique such affine lift conjugate to the Bellman map, and that a coupling-matched teacher, exact population regression, and those independence assumptions make the target unbiased at every bootstrap weight on canonical fields, with a defect off those fields; and \textbf{(iv)}~across toy MRPs, OGBench, and D4RL, \rebf{} improves $W_1$ by up to $7.7\times$ over PCBF.

\begin{figure}[t]
\centering
\includegraphics[width=0.9\linewidth]{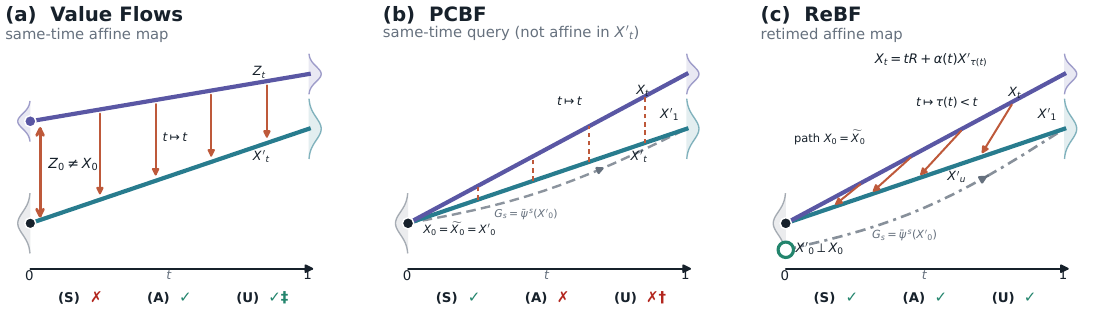}
\caption{\textbf{Three ways to bootstrap a velocity.} (a) Value Flows: same-time affine, but $Z_0\neq X_0$. (b) PCBF: source-consistent and same-time, with a residual in $X_0$. (c) \rebf{}: retimed query, independent noise.}
\label{fig:framework}
\end{figure}

\section{Related Work}

Existing flow-based critics construct distributional value targets through three primary mechanisms. First, endpoint-based and sampled methods avoid pathwise velocity bootstrapping altogether: FLOQ \citep{agrawalla2026} regresses to point targets without modeling return distributions, Bellman Diffusion and distributional flow critics \citep{chen2025,li2024} back up endpoint laws without pathwise alignment, and Temporal Difference Flows \citep{farebrother2025} recovers values only in expectation. Second, path-coupling frameworks leverage intermediate flow fields by incorporating successor velocities \citep{dong2026,xu2026,espinosa2025}; however, they inevitably trade off path consistency against target bias. For instance, Value Flows maintains same-time affinity but violates source consistency, whereas PCBF preserves Gaussian initial noise but suffers from target bias under same-time, shared-noise queries with a self-teacher critic (Theorem~\ref{thm:centering}), as straight-path couplings do not naturally match the teacher's conditional velocity \citep{lipman2023,liu2023,tong2024,lipman2024}. Third, alternative couplings like FlowIQN \citep{groom2026} match quantiles within a single return rather than noise across Bellman steps, preserving Wasserstein contraction only when the update exactly realizes the underlying Bellman operator (Section~\ref{sec:centering-sub}).


\section{Research Gap and Motivation}
\label{sec:gap}
In this section, we describe the limitations of existing flow critics, state the three design properties their targets should satisfy, and show why, on the straight path, no residual-free same-time affine lift with a time-only scale holds for every endpoint pair. On canonical fields, unbiasedness is what affinity and fresh independence produce when the teacher is coupling-matched and the population regression is exact.

\subsection{Preliminaries}
\label{sec:prelim}
A flow critic models the return law $\eta^\pi(s,a)$ of a fixed policy by a velocity $\vth(t,x\mid s,a)$, $t\in[0,1]$. Transition tuple is $H=(R,S',A',\gt)$, with $S'\sim P(\cdot\mid s,a)$, $A'\sim\pi(\cdot\mid S')$, and $0\le\gt\le\gamma$ ($\gt=0$ at termination), and the return satisfies
\begin{equation}
X_1\eqd R+\gt X_1',\qquad X_1'\mid H\sim\eta_{S',A'},
\label{eq:bellman}
\end{equation}
The right-hand side is the Bellman operator $\Tpi$, with unique fixed point $\eta^\pi$~\citep{bellemare2017}. Generation integrates $\vth$ from $X_0\sim\cN(0,I)$; the endpoint law is the critic's model of $\eta^\pi(s,a)$. Flow matching~\citep{lipman2023,liu2023} trains $\vth$ on the straight path
\begin{equation}
X_t=(1-t)X_0+tX_1,\qquad X_0\indep X_1\sim\eta(s,a),\quad t\sim\mathcal U[0,1],
\label{eq:straight_path}
\end{equation}
by regressing $\vth(t,X_t\mid s,a)$ onto $X_1-X_0$. The minimiser is the marginal velocity $v^\star(t,x)=\E[X_1-X_0\mid X_t=x]$, whose flow from $\cN(0,I)$ reproduces $\eta(s,a)$. At $t=0$, $X_t=X_0$ is independent of $X_1$, so $v^\star(0,x)=\E[X_1]-x$: the mean return is one evaluation of the field at $t=0$. Substituting the Bellman return into that path gives the student trajectory and its velocity,
\begin{equation}
X_t = (1-t)X_0 + t\,(R + \gt X_1'), \qquad
Y \defeq \dot X_t = R - (1-\gt)X_0 + \gt\,(X_1'-X_0).
\label{eq:current}
\end{equation}
A backup carries three noises: the student source $X_0\sim\cN(0,I)$, the successor-path start $\Xnoise$, and the generation noise $X_0'$ with $X_1'=\psib^{\,1}(X_0'\mid S',A')$. Following \citet{xu2026} we set $\Xnoise=X_0$ and $X_0'=\rho_1 X_0+\sqrt{1-\rho_1^2}E$ for independent $E\sim\cN(0,I)$ and $\rho_1\in[0,1]$. Velocity bootstrapping replaces the sampled chord $X_1'-X_0$ by a frozen successor critic $\vtg$, queried at time $u(t)$ on $X'_u=(1-u)X_0+uX_1'$. The resulting target is $u_t$.

\subsection{Three properties for velocity bootstrapping}
\label{sec:desiderata}
Source consistency is the requirement on the path. Affinity and decoupled generation are the two construction choices. On canonical fields, conditional unbiasedness is what those two produce when the teacher is coupling-matched and the population regression is exact.

\textbf{\dS{} Source consistency.} The source path and the successor path are the straight interpolants of flow matching, Eq.~\eqref{eq:straight_path}. Flow matching requires those paths to originate from a standard Gaussian at $t=0$, whereas the Bellman operator constrains the return distribution at $t=1$. Absorbing immediate rewards into the initial noise forces the path to start from a shifted distribution $R+\gt\Xnoise$ rather than $X_0\sim\cN(0,I)$, which causes boundary mismatches to propagate along the generated trajectory \citep{xu2026}.

\textbf{\dA{} Affinity.} Every student point $X_t$ is an affine image of a successor point $X'_{u(t)}$, given by $X_t=a(t)R+\varphi(t)X'_{u(t)}$, which reduces to the Bellman map $x\mapsto R+\gt x$ at $t=1$. The teacher query is then available in closed form and visible for every noise coupling. Visibility aligns the student's input with the teacher's query. It does not match the teacher to the coupling, so affinity alone does not make the target unbiased.

\textbf{\dD{} Decoupled generation.} The successor return is generated with fresh noise, $\rho_1=0$, the noise--return pairing of independent flow matching. Only the generation noise $X_0'$ is independent of $X_0$; the student and successor interpolants may still share $X_0$.

\paragraph{When the noise coupling brings no additional bias.}
\label{sec:template}
Replacing the sampled chord by a teacher always leaves an error: the teacher is not the chord. The design question is whether the noise coupling adds a further bias, and whether that addition can be kept at zero. The exact condition is that the relevant conditional residual have zero expectation. Under an exact teacher, two structural conditions are sufficient for that. They are not necessary: cancellation, restricted support, or $\kappa=0$ can make the residual expectation vanish without both.
\begin{enumerate}
    \item \textbf{Query visibility,} from affinity: $X'_{u(t)}$ and $X_t$ determine each other given $H$. If $X_t$ leaves the query uncertain, the student can average over mismatched teacher evaluations.
    \item \textbf{Decoupled generation:} independence holds only at $\rho_1=0$; a value $0<\rho_1<1$ is partial correlation. The self-teacher claim uses fresh independence in the sense of (A5) and (A6), which the label $\rho_1=0$ alone does not give, together with a teacher that is the exact conditional velocity for that pairing. On canonical fields the critic is then its own matched teacher (Lemma~\ref{lem:selfconsistency}).
\end{enumerate}
Retiming $u(t)\to\tau(t)$ supplies the first condition for every coupling. It does not match the teacher; that is decoupled generation. Affinity and decoupled generation together add no coupling bias when the teacher is exact (Theorem~\ref{thm:centering}). A learned teacher still leaves approximation error, which is separate from the bias the coupling itself adds.

\subsection{Two prior paths, and the property each gives up}
\label{sec:existing}\label{sec:vf}\label{sec:pcbf}
\textbf{Value Flows} \citep{dong2026} satisfies \dA{} and gives up \dS{} and \dD{} on its DCFM path. The published loss is a one-to-one mixture of BCFM and DCFM. BCFM matches a path that starts at raw $\cN(0,I)$ noise, so that term is source-consistent on its own. DCFM applies the affine map $Z_t=R+\gt X'_t$ at the same time, with velocity $\dot Z_t=\gt(X_1'-\Xnoise)$ and corrected target $u^{\mathrm{VF}}_t=\gt\,\vtg\big(t,(Z_t-R)/\gt\mid S',A'\big)$. The map is invertible, so the query is visible, and with a matched teacher the corrected target is unbiased. The published DCFM loss omits $\gt$ and regresses on $\vtg$ itself \citep[Eq.~5]{dong2026}; with a perfect teacher that target has conditional mean $\dot Z_t/\gt$. The DCFM path ends at the Bellman target but starts at $Z_0=R+\gt\Xnoise$. Generation always starts from a fresh $\cN(0,I)$ draw, so the DCFM term near $t=0$ is queried on inputs that term rarely saw. The field does see Gaussian inputs near $t=0$, through BCFM; what the mixture fits is two paths with different boundary conditions at $t=0$. Anchoring and confidence weights are added to compensate \citep{dong2026}.

\textbf{PCBF} \citep{xu2026} keeps the source-consistent path of Eq.~\eqref{eq:current} and queries the teacher at the same time $t$, at $X'_t$, with target $R-(1-\gt)X_0+\gt\,\vtg(t,X'_t)$ (full correction, $\kappa=1$ in our parametrisation). The student point and the same-time chord are not affinely related:
\begin{equation}
X_t=tR+\gt X'_t+(1-t)(1-\gt)X_0,\qquad
X'_t=\frac{X_t-tR}{\gt}-\frac{(1-t)(1-\gt)}{\gt}\,X_0 .
\label{eq:pcbf_residual}
\end{equation}
The query is visible only if $X_0$ is. If the generation noise is not fully shared ($\rho_1<1$), $X_0$ is hidden and even an exact teacher leaves a conditioning residual (Eq.~\eqref{eq:sametime-bias}, correlated-Gaussian model). Decoupled generation is the endpoint $\rho_1=0$ of that range, not every $\rho_1<1$. PCBF shares the noise ($\rho_1=1$): for scalar returns $X_0$ is recoverable from $(X_t,H)$, the query is visible, and a teacher matched to that coupling is exact (Proposition~\ref{prop:leak}(i)); the reused critic is generally not that teacher (Proposition~\ref{prop:leak}(ii)), so the target stays biased. PCBF controls the resulting bias and variance with a gain $\lambda$ tuned per domain \citep{xu2026}. So PCBF keeps \dS{} and gives up \dA{} and \dD{}. In this construction the reused critic is not the coupling-matched teacher, so the target stays biased even before further approximation error. Missing either structural condition does not by itself force that bias.

\subsection{The gap: no source-consistent same-time affine link}
\label{sec:triangle}
Figure~\ref{fig:framework} and Table~\ref{tab:triangle} summarise the three methods: \valueflows{} gives up \dS{} and \dD{} on its DCFM path, and PCBF gives up \dA{} and \dD{}. The impossibility itself does not involve the coupling. After removing $tR$, the student point weights $(X_0,X_1')$ as $(1-t,\ \gt t)$, while the same-time chord point weights them as $(1-t,\ t)$, so, on that straight path, no residual-free same-time affine lift with a time-only scale holds for every endpoint pair when $0<\gt<1$ (Corollary~\ref{cor:impossible}). Affinity is recovered by reading the teacher earlier. Together with fresh independence (A6), a coupling-matched teacher, and exact population regression, that alignment is sufficient for unbiasedness on canonical fields.

\begin{corollary}[Impossible Triangle: No same-time affine design under straight paths flow]
\label{cor:impossible}
Write $X'_t=(1-t)X_0+tX_1'$ for the straight successor path read at the same time $t$. The first candidate is the same-time affine transformation $X_t=tR+\alpha(t)X'_t$. For $0<\tilde\gamma<1$ and $0<t<1$, no scale $\alpha(t)$ depending only on $(t,\tilde\gamma)$ makes this identity hold for every endpoint pair $(X_0,X_1')$, while $X_t$ stays the straight source path of Eq.~\eqref{eq:current}.
\end{corollary}
\noindent Its proof is in Appendix~\ref{app:proofs}.

\paragraph{How we close the gap.}
Reading the chord at the earlier time $\tau(t)<t$ whose weights are proportional to $(1-t,\ \gt t)$ removes the residual, so the query is visible under every coupling; fresh independence in the sense of (A5) and (A6), which the label $\rho_1=0$ alone does not give, then makes the critic, on canonical fields, its own matched teacher (Theorem~\ref{thm:centering}(ii)). This is \rebf{}, developed next.

\section{Retimed Bellman Flows}
\label{sec:method}

ReBF keeps the source-consistent path of Eq. (3), so it satisfies (S) by construction, and removes
the same-time residual by reading the successor chord at a different flow time. Section 4.1 obtains
affinity by retiming the teacher query, Section 4.2 writes the resulting target, and Section 5.3 adds
decoupled generation and shows that the two choices together make the target unbiased.

\noindent That is what forces a warped clock, which the next paragraphs derive. The qualifier ``for every endpoint pair'' matters. Under shared noise ($\rho_1=1$) the successor endpoint is a function of $X_0$, and for scalar returns with a generator that is the time-one map of a unique one-dimensional flow, both $X_t$ and $X'_t$ are strictly increasing in $X_0$ given $H$ for $t<1$, so $X'_t$ is a function of $(X_t,H)$: precisely the corner where the two clocks' oracle targets coincide, and the corner PCBF occupies (Corollary~\ref{cor:fixedpoint}).

\textbf{Affine warped transformation.}
The same source path admits the warped-clock rewrite
\begin{equation}
X_t=tR+\alpha(t)X'_{\tau(t)},\qquad X'_{\tau(t)}=(1-\tau(t))X_0+\tau(t)X_1',
\label{eq:affine}
\end{equation}
where the quantity being read is a straight path in the successor endpoint.
The scale $\alpha(t)$ and the successor clock $\tau(t)$ are fixed by matching Eq.~\eqref{eq:affine} to the expansion of Eq.~\eqref{eq:current}.

\subsection{Why the successor clock is slower}
Matching $X_0$ and $X_1'$ coefficients after substituting Eq.~\eqref{eq:affine} into Eq.~\eqref{eq:affine} and expanding Eq.~\eqref{eq:current} with $X_1=R+\tilde\gamma X_1'$ yields $\alpha(t)(1-\tau(t))=1-t$ and $\alpha(t)\tau(t)=\tilde\gamma t$. Adding and dividing gives the clock, and the same relation then fixes the successor query from the current point and the transition:
\begin{equation}
\alpha(t)=1-(1-\tilde\gamma)t,\qquad \tau(t)=\tfrac{\tilde\gamma t}{\alpha(t)},\qquad \xi_t=X'_{\tau(t)}=\tfrac{X_t-tR}{\alpha(t)}.\label{eq:clock}
\end{equation}
For $0<\tilde\gamma<1$ and $0<t<1$, $\tau(t)<t$ because the source path weights $X_1'$ by $\tilde\gamma t$ while a same-time reading weights it by $t$. Once the transition is known, $\xi_t$ is determined by the current regression point.

Theorem~\ref{thm:conjugacy} strengthens Corollary~\ref{cor:impossible} from an impossibility into a characterisation: among residual-free affine lifts with a time-only scale, $(\alpha,\tau)$ is the \emph{only} one under which the velocity update is a Bellman map. Non-affine, state-dependent, and other-path clocks lie outside that statement.

\subsection{Unbiased control variate and time-varying gain}
\textbf{Velocity target.}
Differentiating the affine relation Eq.~\eqref{eq:affine} by the chain rule with Eq.~\eqref{eq:clock} gives the sample velocity
\begin{equation}
Y=\dot X_t=R-(1-\tilde\gamma)\xi_t+g(t)(X_1'-X_0),
\label{eq:identity}
\end{equation}
with gain $g(t)=\alpha(t)\tau'(t)=\tilde\gamma/\alpha(t)$. Replacing the sampled chord $X_1'-X_0$ by a frozen successor teacher $\bar v$, with $a_t=\kappa g(t)$, defines the control variate and the general target
\begin{align}
C_{\tau(t)}&:=\bar v(\tau(t),X'_{\tau(t)}\mid H)-(X_1'-X_0),\label{eq:cv}\\
u_a&=Y+a_t\,C_{\tau(t)}.\label{eq:general-target}
\end{align}
At weight $\kappa$ the regression target is
\begin{equation}
u_\kappa=Y+\kappa g(t)\,C_{\tau(t)}=R-(1-\tilde\gamma)\xi_t+g(t)\{(1-\kappa)(X_1'-X_0)+\kappa\bar v(\tau(t),\xi_t\mid H)\}.
\label{eq:target}
\end{equation}

With a matched teacher, $\E[C_{\tau(t)}\mid X_t,H]=0$, so every $\kappa$ keeps the same conditional mean (Theorem~\ref{thm:centering}). The weight $\kappa\in[0,1]$ trades residual bridge noise against teacher error; $\kappa=1$ is the partial Rao--Blackwellization of Proposition~\ref{prop:risk}. Replacing a sampled quantity by its conditional expectation is Rao--Blackwellization and the gain is a control-variate weight; neither is new, and we do not claim either as a contribution. The gain $g(t)=\tilde\gamma/\alpha(t)$ rises from $\tilde\gamma$ at $t=0$ to $1$ at $t=1$, so the weight $\kappa g(t)$ varies with $t$, while PCBF uses a fixed gain $\lambda=\tilde\gamma$. For terminal backups we use $Y=R-X_0$ and set the correction to zero, avoiding the singular $\tilde\gamma=0,t=1$ inverse.

\subsection{Decoupled Generation}
\label{sec:coupling}\label{sec:pcbf-relation}
With $\tilde X_0=X_0$ already fixed, the generation noise is
\begin{equation}
X_0' = \rho_1 X_0 + \sqrt{1-\rho_1^2}\,E, \label{eq:coupling}
\end{equation}
where $E\sim\mathcal N(0,\Id_d)$ is independent of $(X_0,H,t)$, $0\le\rho_1\le 1$, and $X_1'=\psib^{\,1}(X_0'\mid S',A')$. By default ReBF sets $\rho_1=0$: the successor is drawn independently, and the backup is still the full correction. PCBF sets $\rho_1=1$, tying $X_0'$ to the path noises. That tie can cancel a same-time bias, but it is stronger than an unbiased bootstrap needs. Section~\ref{sec:toys} records its cost. 
The matched teacher on the retimed path is
\begin{equation}v_{\rho_1}^\star(\tau(t),b\mid H)=\E[X_1'-X_0\mid X'_{\tau(t)}=b,H].\label{eq:oracle}\end{equation}
At $\rho_1=1$, when the chord is recoverable from the successor query, the two clocks share this oracle. ReBF is the retimed clock with $\rho_1=0$ and $\kappa=1$ (Algorithm~\ref{alg:sw}); PCBF is the same-time clock with $\rho_1=1$ and $\kappa=1$. The uncorrected critic is $\kappa=0$, whose clock is inert.

\section{Theory}
\label{sec:theory}\label{sec:operator}
\begin{definition}[Bellman velocity operator]
\label{def:operator}
Fix $\pi$ and the source $X_0\sim\mathcal N(0,\Id_d)$. For a velocity field $v$, let $\mathcal E(v)_c$ be the law at $t=1$ of the flow $\dot Z_t=v_c(Z_t,t)$, $Z_0=X_0$: the returns the critic \emph{generates}. For a return-law family $\eta$, let $(\T^\pi\eta)_c=\Law(R+\tilde\gamma X_1'\mid c)$ with $X_1'\sim\eta_{S',A'}$, as in Eq.~\eqref{eq:bellman}, and let $(\mathcal L\eta)_c(z,t)=\E[X_1-X_0\mid(1-t)X_0+tX_1=z]$ with $X_1\sim\eta_c$ independent of $X_0$: the straight flow-matching field of $\eta_c$. The \emph{Bellman velocity operator} is
\begin{equation}
\mathcal C^\pi=\mathcal L\,\T^\pi\,\mathcal E:\qquad
v\ \xrightarrow{\ \mathcal E\ }\ \text{returns}\ \xrightarrow{\ \T^\pi\ }\ \text{backed-up returns}\ \xrightarrow{\ \mathcal L\ }\ \text{new field}.
\label{eq:pop-operator}
\end{equation}
A field is \emph{canonical} if it equals $\mathcal L\eta$ for some $\eta$. Its \emph{readout} $q_v(z,t)=z+(1-t)v(z,t)$ is the endpoint a particle at $z$ reaches if it keeps its velocity; for a canonical field and $t<1$ this is $\E[X_1\mid X_t=z]$.
\end{definition}
\noindent Generate returns with the critic, apply the Bellman backup, and fit the straight flow of the result. Assumptions (A1)--(A9) and the full statements are in Appendix~\ref{app:full}.

\paragraph{On one backup, retiming is the Bellman map.}
\label{sec:conjugacy}
Freeze one draw $h=(r,s',a',\tilde\gamma)$ of the backup, a \emph{branch}; on it the Bellman equation is $X_1=r+\tilde\gamma X_1'$.
\begin{theorem}[Bellman conjugacy characterises the retiming]
\label{thm:conjugacy}
Fix a nonterminal branch with $0<\tilde\gamma<1$, and let $v'$ be the successor's flow-matching field under the branch's coupling. Among affine lifts $X_t=tr+\varphi(t)X'_{\vartheta(t)}$ with $\varphi(0)=1$ and $\vartheta(0)=0$, whose field we write $\B^{\varphi,\vartheta}v$, the readout identity
\begin{equation}
q_{\B^{\varphi,\vartheta}v}(z,t)=r+\tilde\gamma\,q_v\big(\xi,\vartheta(t)\big),\qquad \xi=\frac{z-tr}{\varphi(t)},
\label{eq:conj-short}
\end{equation}
holds for every field $v$ if and only if $(\varphi,\vartheta)=(\alpha,\tau)$ of Eq.~\eqref{eq:clock}; a same-time query forces $\tilde\gamma=1$. The backed-up field is then
\begin{equation}
(\B_{r,\tilde\gamma}v')(z,t)=r-(1-\tilde\gamma)\,\xi+g(t)\,v'(\xi,\tau),\qquad g=\tilde\gamma/\alpha,
\label{eq:operator}
\end{equation}
which for canonical $v'$ is $\mathcal C^\pi$ on that branch, and it shrinks readout differences by exactly $\tilde\gamma$.
\end{theorem}
\noindent In readouts the backup is $x\mapsto r+\tilde\gamma x$. The same-time query is not among these lifts unless $\tilde\gamma=1$.

\paragraph{Retiming makes the sampled target unbiased for the operator.}
\label{sec:centering-sub}\label{sec:fixedpoint}
The algorithm regresses on the sampled target $u_\kappa$ of Eq.~\eqref{eq:target}, which replaces a fraction $\kappa$ of the chord $X_1'-X_0$ by the teacher at the retimed point.
\begin{theorem}[Centering and realisation]
\label{thm:centering}
Assume (A2)--(A4): a coupling-matched teacher, a gain measurable given $(X_t,H)$, and a square-integrable target.
\emph{(i)} $\E[C_{\tau(t)}\mid X_t,H]=0$, so $\E[u_\kappa\mid X_t]=v_t^\star(X_t\mid c)$ for every gain.
\emph{(ii)} With fresh independence (A6) --- the label $\rho_1=0$ alone does not give it --- and the critic's own field as teacher, the regression target equals $\mathcal C^\pi v$ on canonical fields. Off those fields a defect remains; its formula is Theorem~\ref{thm:centering-full}(ii).
\emph{(iii)} Training then iterates $\mathcal C^\pi$.
\end{theorem}
\noindent Part (i) is the tower property: the retimed query is a function of $(X_t,H)$, so replacing the chord by its conditional mean leaves the average unchanged. The same-time query is not, and its residual, Eq.~\eqref{eq:sametime-bias}, is scaled by $\kappa$.

\begin{corollary}[Fixed point and contraction]
\label{prop:operator}
Under (A1), $\mathcal E\,\mathcal C^\pi=\T^\pi\mathcal E$, so $\mathcal C^\pi$ has the unique fixed point $v^\pi=\mathcal L\eta^\pi$, and the returns two critics generate contract at the Bellman rate:
\begin{equation}
\delta_p(\mathcal C^\pi v,\mathcal C^\pi w)\le\beta_p\,\delta_p(v,w),\qquad
\delta_p\big((\mathcal C^\pi)^nv,v^\pi\big)\le\beta_p^{\,n}\,\delta_p(v,v^\pi),
\label{eq:contract}
\end{equation}
with $\beta_p=\max_c\E[\tilde\gamma^p\mid c]^{1/p}\le\gamma$ and $\delta_p$ the maximum $W_p$ distance between generated return laws. A method inherits this only when its target estimates $\mathcal C^\pi v$, which Theorems~\ref{thm:conjugacy} and~\ref{thm:centering} give for the retimed clock. The full statement is Proposition~\ref{prop:operator-full}.
\end{corollary}

The gain trades residual bridge noise $S$ against teacher error $e$. The conditional risk is minimised at $\kappa^\star=S/(S+\norm e^2)$ when the denominator is positive, so full correction is optimal exactly for an exact teacher (Proposition~\ref{prop:risk-full}).\label{prop:risk}

\label{sec:shared}\label{sec:regime}

\begin{table}[h]
\centering
\caption{\textbf{$W_1$, mean, and std.\ against comparators} (lower is better; mean bias near zero). Bold: best in the row; underlined: not different from the best ($p\ge0.05$). std.\ is $100\sum_s|\mathrm{sd}_s-\mathrm{sd}^\star_s|/\sum_s\mathrm{sd}^\star_s$ over states. $^\dagger$over-dispersed; $^\P$per-coordinate floq; $^\ddagger$published Value Flows; $^\S$full maze, $2{,}880$ points.}
\label{tab:compare}
\footnotesize
\setlength{\tabcolsep}{2.7pt}
\begin{tabular}{@{}llcccccc@{}}
\toprule
& & BCFM & PCBF & PCBF, tuned $\kappa$ & floq & Value Flows & \textbf{ReBF (ours)}\\
Env & Metric & {\scriptsize(same-time, 1, 0)} & {\scriptsize(same-time, 1, 1)} & {\scriptsize(same-time, 1, $\kappa^\star$)} & {\scriptsize scalar} & {\scriptsize default} & {\scriptsize(retimed, 0, 1)}\\
\midrule
\textsc{bernoulli} & $W_1$ & \underline{0.0181}\,{\scriptsize$\pm$.0023} & 0.0493\,{\scriptsize$\pm$.0022} & 0.0180\,{\scriptsize$\pm$.0022} & 0.184\,{\scriptsize$\pm$.001} & 0.122\,{\scriptsize$\pm$.004} & \textbf{0.0151}\,{\scriptsize$\pm$.0017}\\
 & mean & \underline{$+0.38$}\,{\scriptsize$\pm$.55} & \underline{$+0.36$}\,{\scriptsize$\pm$.49} & \underline{$+0.35$}\,{\scriptsize$\pm$.55} & {\boldmath$-0.42$}\,{\scriptsize$\pm$.35} & $+12.2$\,{\scriptsize$\pm$.4} & \underline{$-0.17$}\,{\scriptsize$\pm$.44}\\
 & std. & \textbf{1.2}\,{\scriptsize$\pm$.2} & 7.2\,{\scriptsize$\pm$.4} & \underline{1.4}\,{\scriptsize$\pm$.2} & 13.2\,{\scriptsize$\pm$.2} & 3.7$^\dagger$\,{\scriptsize$\pm$.3} & \underline{1.7}\,{\scriptsize$\pm$.3}\\
\addlinespace[2pt]
\textsc{solitaire} & $W_1$ & 0.147\,{\scriptsize$\pm$.006} & 0.677\,{\scriptsize$\pm$.046} & 0.154\,{\scriptsize$\pm$.007} & 1.464\,{\scriptsize$\pm$.010} & 1.334\,{\scriptsize$\pm$.010} & \textbf{0.0878}\,{\scriptsize$\pm$.0096}\\
 & mean & {\boldmath$-0.87$}\,{\scriptsize$\pm$.36} & $-14.0$\,{\scriptsize$\pm$1.2} & \underline{$-0.94$}\,{\scriptsize$\pm$.42} & $+12.6$\,{\scriptsize$\pm$.6} & $-39.0$\,{\scriptsize$\pm$.4} & \underline{$-1.73$}\,{\scriptsize$\pm$.46}\\
 & std. & 4.5\,{\scriptsize$\pm$.4} & 15.8\,{\scriptsize$\pm$1.3} & 4.8\,{\scriptsize$\pm$.5} & 35.3\,{\scriptsize$\pm$.3} & 49.4\,{\scriptsize$\pm$.3} & \textbf{2.3}\,{\scriptsize$\pm$.3}\\
\addlinespace[2pt]
\textsc{discrete-} & $W_1$ & 0.517\,{\scriptsize$\pm$.015} & 0.617\,{\scriptsize$\pm$.016} & 0.493\,{\scriptsize$\pm$.018} & 1.950\,{\scriptsize$\pm$.015} & 6.903\,{\scriptsize$\pm$.013} & \textbf{0.298}\,{\scriptsize$\pm$.008}\\
\textsc{mc} & mean & \underline{$-0.80$}\,{\scriptsize$\pm$.26} & $+0.04$\,{\scriptsize$\pm$.50} & \underline{$+0.37$}\,{\scriptsize$\pm$.35} & {\boldmath$-0.04$}\,{\scriptsize$\pm$.19} & $-53.8$\,{\scriptsize$\pm$.1} & \underline{$-0.04$}\,{\scriptsize$\pm$.24}\\
 & std. & 8.4\,{\scriptsize$\pm$.4} & 11.4\,{\scriptsize$\pm$.4} & 8.7\,{\scriptsize$\pm$.4} & 56.3\,{\scriptsize$\pm$.1} & 62.4\,{\scriptsize$\pm$.1} & \textbf{3.7}\,{\scriptsize$\pm$.2}\\
\addlinespace[2pt]
\shortstack[l]{four-rooms\\{\scriptsize(16-dim SF)}} & $W_1$ & 0.0523\,{\scriptsize$\pm$.0007} & 0.0410\,{\scriptsize$\pm$.0010} & 0.0394\,{\scriptsize$\pm$.0008} & 0.372\,{\scriptsize$\pm$.001}$^\P$ & 0.241\,{\scriptsize$\pm$.002} & \textbf{0.0282}\,{\scriptsize$\pm$.0009}\\
 & mean & \underline{2.9}\,{\scriptsize$\pm$.1} & \underline{3.0}\,{\scriptsize$\pm$.1} & 2.7\,{\scriptsize$\pm$.1} & 9.1\,{\scriptsize$\pm$.2}$^\P$ & 29.4\,{\scriptsize$\pm$.2} & \textbf{2.7}\,{\scriptsize$\pm$.1}\\
 & std. & 10.6\,{\scriptsize$\pm$.2} & 8.1\,{\scriptsize$\pm$.2} & 7.6\,{\scriptsize$\pm$.2} & 89.8\,{\scriptsize$\pm$.1}$^\P$ & 20.6$^\dagger$\,{\scriptsize$\pm$.1} & \textbf{2.8}\,{\scriptsize$\pm$.1}\\
\addlinespace[2pt]
\textsc{ogbench} & $W_1$ & 22.6\,{\scriptsize$\pm$.3} & 19.5\,{\scriptsize$\pm$.5} & \underline{19.9}\,{\scriptsize$\pm$.7} & 28.3\,{\scriptsize$\pm$.4} & 60.4\,{\scriptsize$\pm$.6}$^\ddagger$ & \textbf{14.6}\,{\scriptsize$\pm$1.0}\\
teleport$^\S$ & mean & $+7.7$\,{\scriptsize$\pm$.6} & $+10.4$\,{\scriptsize$\pm$1.2} & $+9.9$\,{\scriptsize$\pm$1.2} & {\boldmath$+0.4$}\,{\scriptsize$\pm$.7} & $+40.4$\,{\scriptsize$\pm$.5}$^\ddagger$ & \underline{$+4.1$}\,{\scriptsize$\pm$1.2}\\
 & std. & 54.2\,{\scriptsize$\pm$.6} & \underline{21.6}\,{\scriptsize$\pm$2.0} & \underline{33.6}\,{\scriptsize$\pm$7.6} & 92.5\,{\scriptsize$\pm$.3} & 92.1\,{\scriptsize$\pm$.4}$^\ddagger$ & \textbf{19.6}\,{\scriptsize$\pm$4.6}\\
\bottomrule
\end{tabular}

\end{table}

\begin{figure}[t]
\centering
\includegraphics[width=\linewidth]{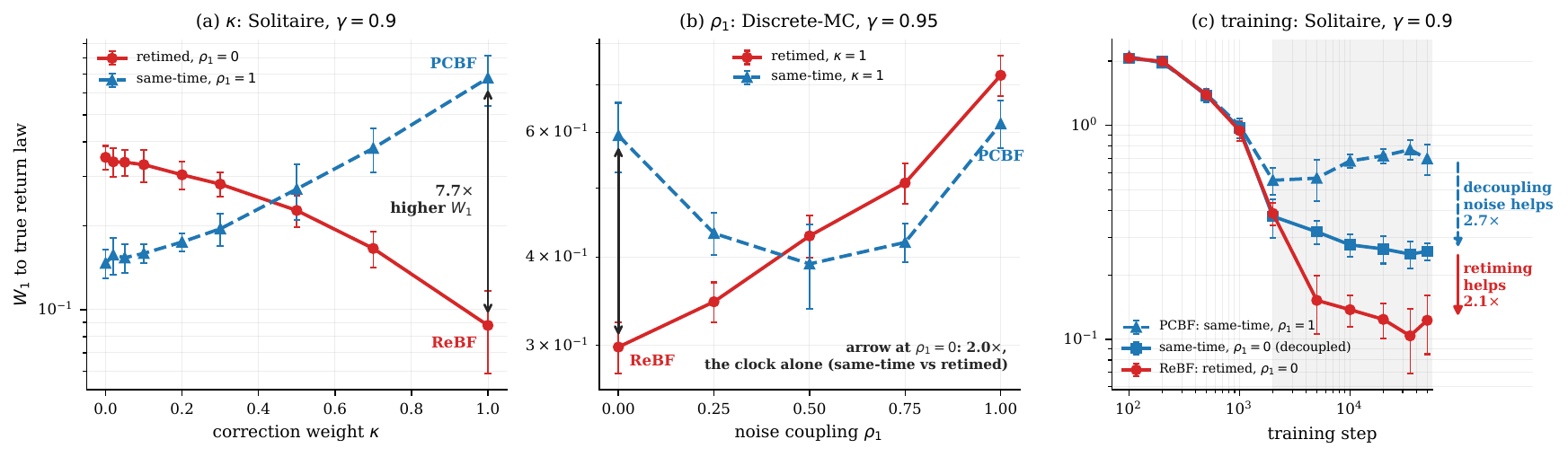}
\caption{\textbf{Why \rebf{} wins.} 
(a) Effect of correction weight $\kappa$ across clock and coupling variants. 
(b) Effect of noise coupling $\rho_1$ at full correction ($\kappa=1$). 
(c) Training dynamics isolating the sequential gains of noise decoupling and clock retiming; Details in Appendix~\ref{app:allenvs}).}
\label{fig:main}
\end{figure}

\section{Experiments}\label{sec:toys}
We structure our empirical evaluation around three core questions: whether \rebf{} yields more accurate return laws, whether query-time retiming and noise decoupling explain this advantage, and whether this superior fidelity translates to downstream offline control. For policy evaluation, we benchmark under fixed architectures and budgets on three analytical MRPs (\textsc{bernoulli}, \textsc{solitaire}, \textsc{discrete-mc}, details in \ref{app:envs}), successor features in \textsc{four-rooms} (\ref{app:env-fourrooms}), and \textsc{pointmaze-teleport}, measuring $W_1$ distance against ground-truth or rollout-based reference distributions. For downstream control, we evaluate \rebf{} across 38 OGBench and D4RL tasks and the results are shown in Table \ref{tab:aggregated}. Overall, \rebf{} achieves \textbf{state-of-the-art performance}—most prominently on long-horizon, multi-stage manipulation tasks where uncancelled target bias severely degrades traditional value estimation.


\paragraph{Does ReBF recover the return law?}
It has the lowest $W_1$ and dispersion error on all benchmarks (Table~\ref{tab:compare}). Most methods other than Value Flows get the mean nearly right, so what separates them is the shape of the law: a correct mean does not certify a correct law (Figure~\ref{fig:flowmaps}). PCBF's best $\kappa$ weakens the correction toward the uncorrected BCFM, whereas ReBF is best at full correction in every environment (Figure~\ref{fig:main}a, Table~\ref{tab:headline}). This is what the theory predicts: ReBF's correction is unbiased for every $\kappa$, so nothing is gained by shrinking it (Theorem~\ref{thm:centering}), whereas PCBF's inherits the error of its mismatched teacher, scaled by $\kappa$ (Proposition~\ref{prop:leak}). 

\label{sec:ogbench}
\begin{table*}[t]
\centering
\caption{\textbf{Offline RL Results.} We report performance on OGBench and D4RL. Results are averaged over 8 random seeds (4 seeds for pixel-based tasks). Bold numbers denote the values within 95\% of the best performing method on each domain. Per-task results are in Table~\ref{tab:fullogbench_comparison}.}
\label{tab:aggregated}
\resizebox{\textwidth}{!}{%
\begin{tabular}{lcccccccc}
\toprule
Domain & IQN & CODAC & FloQ & FQL & IQL & Value Flows & PCBF & \textbf{ReBF (Ours)} \\
\midrule
cube-double-play (5 tasks) & $42 \pm 8$ & $61 \pm 6$ & $47 \pm 14$ & $29 \pm 6$ & $6 \pm 2$ & $69 \pm 4$ & $70 \pm 7$ & $\mathbf{76 \pm 6}$ \\
scene-play (5 tasks) & $40 \pm 1$ & $55 \pm 1$ & $\mathbf{58 \pm 4}$ & $56 \pm 2$ & $28 \pm 3$ & $\mathbf{59 \pm 4}$ & $\mathbf{58 \pm 3}$ & $\mathbf{60 \pm 4}$ \\
puzzle-4x4-play (5 tasks) & $27 \pm 4$ & $20 \pm 18$ & $28 \pm 6$ & $17 \pm 5$ & $7 \pm 2$ & $27 \pm 4$ & $32 \pm 5$ & $\mathbf{36 \pm 8}$ \\
cube-triple-play (5 tasks) & $6 \pm 0$ & $2 \pm 1$ & $8 \pm 3$ & $4 \pm 2$ & $1 \pm 1$ & $\mathbf{14 \pm 3}$ & $6 \pm 3$ & $12 \pm 4$ \\
D4RL adroit (8 tasks) & $66 \pm 5$ & $\mathbf{69 \pm 1}$ & $\mathbf{70 \pm 3}$ & $\mathbf{71 \pm 4}$ & $\mathbf{70}$ & $65 \pm 2$ & $\mathbf{70 \pm 2}$ & $\mathbf{69 \pm 1}$ \\
visual-antmaze-teleport (5 tasks) & $4 \pm 2$ & -- & -- & $5 \pm 2$ & $6 \pm 4$ & $13 \pm 4$ & $\mathbf{14 \pm 4}$ & $13 \pm 3$ \\
visual-cube-double-play (5 tasks) & $1 \pm 0$ & -- & -- & $6 \pm 1$ & $11 \pm 6$ & $\mathbf{13 \pm 2}$ & $3 \pm 1$ & $3 \pm 1$ \\
\bottomrule
\end{tabular}%
}
\end{table*}

\begin{figure}[ht]
\centering
\includegraphics[width=0.8\linewidth]{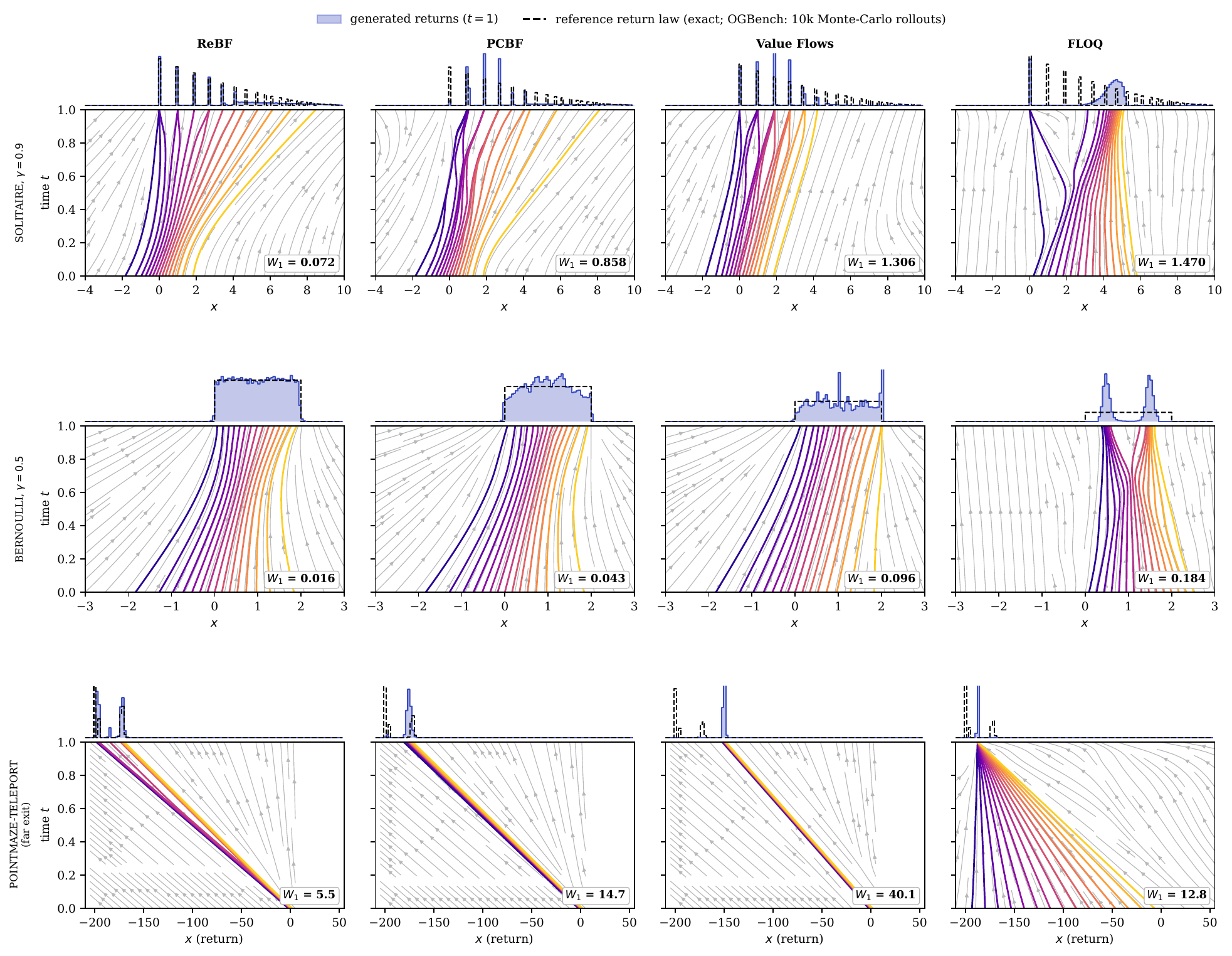}
\caption{\textbf{Comparison of learned flow fields and generated return distributions.} Each plot shows velocity streamlines, sample trajectories from flow time $t=0$ to $t=1$, and the resulting terminal return distributions (purple bars) compared against ground-truth Monte-Carlo rollouts (dashed lines) for ReBF, PCBF, Value Flows, and FLOQ.}
\label{fig:flowmaps}
\end{figure}

\paragraph{Which design makes ReBF the best?}
Neither choice alone: ReBF needs both the retimed clock and independent noise. At full correction, retiming lowers $W_1$ with independent noise ($\rho_1=0$) but not with shared noise ($\rho_1=1$), in all three toys (Figures~\ref{fig:main}b and~\ref{fig:allenvs}b). As Section~\ref{sec:coupling} predicts, for scalar returns shared noise already makes the same-time target exact under a matched teacher, so retiming has nothing to fix there and can help only with independent noise.


\paragraph{Are both retiming and decouple generation necessary?}
Yes: clock retiming ($\tau(t)$) and independent noise coupling ($\rho_1=0$) act strictly synergistically. As shown in Figure~\ref{fig:main}b, retiming significantly lowers $W_1$ error under independent noise, whereas under shared noise ($\rho_1=1$), scalar-return chords are recoverable from either query, causing both clocks to target similar oracle fields. Conversely, maintaining a same-time clock under independent noise remains suboptimal. Isolating their contributions on \textsc{solitaire} demonstrates two additive steps: decoupling noise lowers PCBF's final $W_1$ error by $2.7\times$, and retiming the clock provides a further $2.1\times$ drop (Figure~\ref{fig:main}c).

\paragraph{Why does same-time shared-noise bootstrapping plateau?}
Reusing a critic as its own teacher under shared noise ($\rho_1=1$) introduces an irreducible target bias that corrupts iterative updates. On \textsc{solitaire}, PCBF's error initially decreases but subsequently rises and plateaus at a biased level—a degradation shared by the retimed clock under $\rho_1=1$ (Figure~\ref{fig:main}c). Theoretically, a self-teacher critic under shared noise fails to match the required target field (Proposition~\ref{prop:leak}(ii)), continually feeding conditioning bias back into the Bellman backup. Independent noise ($\rho_1=0$) resolves this mismatch because the required teacher is precisely the marginal field the critic already learns (Lemma~\ref{lem:selfconsistency}).

\paragraph{Conclusion.}
We propose Retimed Bellman Flows (ReBF), a hyperparameter-free framework that queries the teacher critic at a dynamically shifted flow time $\tau(t)$. By restoring exact affine conjugacy with the Bellman operator, ReBF eliminates target bias without extra evaluations. Empirically, ReBF improves ground-truth return fidelity $W_1$ by up to $7.7\times$ on MRPs and achieves state-of-the-art performance across 38 OGBench and D4RL offline benchmarks.

\section{Limitations}
The analysis is for policy evaluation on small Markov reward processes under a fixed policy. Table~\ref{tab:aggregated} reports offline control on OGBench and D4RL; those scores are not a test of the theory. We have not studied how a flow critic should be used to extract a better policy in offline reinforcement learning. Extending the same target to other settings, including offline goal-conditioned reinforcement learning, needs further investigation. Only one of the three environments shows a large shared-noise mean bias, and we have marked which conclusions rest on it. Several results are proved in a correlated-Gaussian model rather than in general: the necessity half of the exactness set, the same-time closed form, and the scale-error identity; the last is also checked on a non-Gaussian law (Table~\ref{tab:regime}). We prove convergence on canonical fields with independent generation (Theorem~\ref{thm:centering}(iii)); we do not establish uniqueness or convergence of the self-consistent training map from non-canonical initial fields when the critic is its own teacher at $\kappa>0$, nor, from any initial field, when generation is coupled ($\rho_1>0$). The mean bias we measure under shared noise has a candidate population mechanism whose size does not match the measurement, and we have not shown that it operates in training.
\label{main:end}

\subsection*{AI use statement}
Large language model assistance was used in preparing this manuscript: for drafting and organizing prose, for suggesting proof strategies, and for implementing the numerical diagnostic scripts. The research questions, original research ideas, the theoretical framework, and the final form of every claim are the authors' own. The authors have verified each theorem statement and proof independently of the assistance received and have confirmed every citation against its primary source. The authors take full responsibility for all content of this paper, including any errors.

\subsection*{Reproducibility statement}
The analytic identity grids and the contraction moduli  run on CPU and require only NumPy, SciPy and SymPy, with Matplotlib for the figures. The Gaussian column of Table~\ref{tab:regime} is Eq.~\eqref{eq:scalefree} in closed form; its Bernoulli and Retimed columns are fixed-seed Monte Carlo estimates ($2\times10^7$ draws, seed $0$), shipped with the supplement. The toy Markov reward processes, arm definitions, training configuration, preregistered analysis protocol and per-seed scores for Section~\ref{sec:toys} are given in Appendix~\ref{app:protocol}; every reported $W_1$ is computed against the analytic return law rather than a sampled reference, and the reference-error floor for each environment is reported alongside it. Appendix~\ref{app:full} states all assumptions and Appendix~\ref{app:proofs} gives complete proofs.

\clearpage
\appendix

Appendix~\ref{app:full} states the assumptions and the full results of Section~\ref{sec:theory}, and Appendix~\ref{app:proofs} proves them. The later appendices give the toy comparisons, environments, protocol, algorithm and limitations. The integration budget and the per-task OGBench table are an online supplement.

\section{Standing Assumptions and Full Statements of Section~\ref{sec:theory}}
\label{app:full}
Section~\ref{sec:theory} states its results briefly. This appendix gives the assumptions they use and their full statements; proofs are in Appendix~\ref{app:proofs}.

\paragraph{Standing assumptions.}
\begin{itemize}
\item[(A1)] \emph{Regularity.} The condition set is finite, $0\le\tilde\gamma\le\bar\gamma<1$ and $1\le p<\infty$; rewards and return laws have uniformly finite $p$th moments; a field is admissible when its flow from $\mathcal N(0,\Id_d)$ is unique on $[0,1)$ with a terminal law of finite $p$th moment; and exact flow matching gives $\mathcal E\mathcal L=\mathrm{id}$.
\item[(A2)] \emph{Teacher.} For centering, the teacher is the coupling-matched oracle of Eq.~\eqref{eq:oracle}; for realisation, it is the critic's own field, with the successor endpoint drawn from $\mathcal E(v)_{S',A'}$ independently of $(X_0,R,\tilde\gamma,t)$ given $(S',A')$.
\item[(A3)] \emph{Gain.} $a_t=\kappa g(t)$ with $\kappa\in[0,1]$ measurable with respect to $\sigma(X_t,H)$; it may not inspect residual chord randomness. A gain that does can destroy centering: if $(X_1'-X_0)-D$ is $+1$ or $-1$ equiprobably given $\HH$ and $\kappa=\mathbf 1\{(X_1'-X_0)-D>0\}$, then $\E[\kappa g\{D-(X_1'-X_0)\}\mid\HH]=-g/2\ne0$.
\item[(A4)] \emph{Integrability.} The regression target is square-integrable.
\item[(A5)] \emph{Sampling and conditioning.} Given the current condition $c$, the sampler draws the backup $H=(R,S',A',\tilde\gamma)$, the training time $T\sim\mathcal U(0,1)$ and the source $X_0\sim\mathcal N(0,\Id_d)$; $X_0$ is independent of $H$, and $T$ is independent of $(X_0,X_1',H)$ given $c$. Conditional expectations are regular versions, jointly measurable in time, query and branch; in particular the oracle of Eq.~\eqref{eq:oracle} is a version $v^\star_{\rho_1}(u,b\mid h)$, jointly measurable in $(u,b,h)$, that is for each $h$ a conditional expectation under the conditional law of $(X_0,X_1')$ given $H=h$. At fixed $(c,t)$ we write $\F=\sigma(X_t)$ and $\HH=\sigma(X_t,H)$, as in Eq.~\eqref{eq:rb}.
\item[(A6)] \emph{Independent fresh generation.} The successor endpoint satisfies $X_1'\sim\mathcal E(v)_{S',A'}$ and $X_1'$ is independent of $(X_0,R,\tilde\gamma,T)$ given $(S',A')$. This holds, for example, when $\rho_1=0$ in Eq.~\eqref{eq:coupling}, $E$ is independent of $(X_0,H,T)$, and $X_1'$ is the time-one map of the flow of $v_{S',A'}$ applied to $E$. The label $\rho_1=0$ alone does not imply (A6).
\item[(A7)] \emph{Bridge flow validity.} Whenever an endpoint law is read off a bridge field, that is, off the conditional-mean velocity of a pair $(X_0,X_1)$ under any coupling, the field is admissible in the sense of (A1) and the continuity equation of the bridge marginals has a unique solution among curves of probability laws with finite $p$th moments. For the independent canonical lift this is the content of $\mathcal E\mathcal L=\mathrm{id}$ in (A1). Differentiation under an expectation is used only for compactly supported smooth test functions or under a displayed integrable dominating function.
\item[(A8)] \emph{Affine clocks.} In Theorem~\ref{thm:conjugacy-full}, $\varphi$ and $\vartheta$ are absolutely continuous on $[0,1)$, $\vartheta$ takes values in $[0,1]$, $\varphi(0)=1$ and $\vartheta(0)=0$; the lifted field is the velocity of the lifted path. That $\varphi>0$ and $\vartheta<1$ on $[0,1)$ is a consequence, proved there, not a hypothesis.
\item[(A9)] \emph{Terminal convention.} If $\tilde\gamma=0$, the target is $Y=R-X_0$ with zero correction, and no inverse involving $\alpha(1)=0$ is evaluated. Clock identities are stated for $0<\tilde\gamma<1$.
\end{itemize}
(A5) and (A9) describe the algorithm's sampler and terminal convention (Section~\ref{sec:method}, Appendix~\ref{app:algo}) and hold throughout; each of (A6)--(A8) is invoked only where a statement cites it.

The statements concern one frozen Bellman regression update: the policy, endpoint generator, teacher and coupling sampler are fixed while taking expectations. Source seeds are independent of $H$, and conditional on $H$ the dependence of $X_1'$ on $X_0$ is exactly the declared coupling. Squared-error statements require finite second moments, and the $W_p$ contraction requires finite $p$th moments. We do not infer these regularity conditions from finite neural-network training.

Assumption (A1), rather than finite moments alone, is what asserts $\mathcal E\mathcal L=\mathrm{id}$, and it includes atomic return laws. If $X_1$ has a finite atomic law and $X_0$ is an independent Gaussian source, then for $t<1$ the interpolant $(1-t)X_0+tX_1$ has a positive smooth density, the weak continuity-equation calculation of Lemma~\ref{lem:flow-matching} applies, and the interpolant converges to $X_1$ in $L^p$ as $t\uparrow1$. Finite atomic toy laws therefore satisfy the endpoint part once the uniqueness clause of (A1) and (A7) is imposed.

\begin{proposition*}[Fixed point and contraction; full statement]{prop:operator}
\label{prop:operator-full}
Assume (A1).
\emph{(i)} $\mathcal E\,\mathcal C^\pi=\T^\pi\mathcal E$, $\mathcal C^\pi\mathcal L=\mathcal L\,\T^\pi$ and, for every integer $n\ge1$, $(\mathcal C^\pi)^n=\mathcal L(\T^\pi)^n\mathcal E$. Every fixed point of $\mathcal C^\pi$ is canonical and its endpoint law is fixed by $\T^\pi$, so $v^\pi=\mathcal L\eta^\pi$ is the unique fixed point.

\emph{(ii)} On canonical fields put $D_p(\mathcal L\eta,\mathcal L\nu)=\max_cW_p(\eta_c,\nu_c)$; on admissible fields put $\delta_p(v,w)=\max_cW_p(\mathcal E(v)_c,\mathcal E(w)_c)$, a pseudometric only, since distinct fields can share an endpoint law. Let $\beta_p=\max_c\E[\tilde\gamma^p\mid c]^{1/p}\le\bar\gamma$. Whenever every multiplier is at most the nominal discount $\gamma$, as for $\tilde\gamma\in\{0,\gamma\}$, also $\beta_p\le\gamma$. For admissible $v,w$ and $n\ge1$,
\[
D_p(\mathcal C^\pi v,\mathcal C^\pi w)\le\beta_p\,\delta_p(v,w),\qquad
D_p\big((\mathcal C^\pi)^nv,v^\pi\big)\le\beta_p^{\,n}\,\delta_p(v,v^\pi),
\]
and on canonical fields $\delta_p=D_p$. Under $\tilde\gamma\le\bar\gamma$ alone the constant $\bar\gamma$ cannot be improved.
\end{proposition*}

\begin{theorem*}[Bellman conjugacy characterises the affine retiming; full statement]{thm:conjugacy}
\label{thm:conjugacy-full}
Fix a nonterminal branch with $0<\tilde\gamma<1$ and its coupling, and let $v'$ be the successor bridge's flow-matching field under that same coupling. Among affine lifts $X_t=tr+\varphi(t)X'_{\vartheta(t)}$ normalised by $\varphi(0)=1,\vartheta(0)=0$, with $\varphi,\vartheta$ absolutely continuous on $[0,1)$ and $\vartheta$ valued in $[0,1]$ as in (A8), read through the endpoint readout $q_v$ of Definition~\ref{def:operator}:

\emph{(i)} with $\xi=(z-tr)/\varphi(t)$ and $\B^{\varphi,\vartheta}v$ the field of the lifted bridge, the conjugacy $q_{\B^{\varphi,\vartheta}v}=r+\tilde\gamma\,q_v\circ(\xi,\vartheta)$ holds for every field $v$ \emph{if and only if} $(\varphi,\vartheta)=(\alpha,\tau)$ of Eq.~\eqref{eq:clock}; an interior same-time query forces $\tilde\gamma=1$.

\emph{(ii)} For that unique pair, with $\Phi_{r,\tilde\gamma}(z,t)=\big((z-tr)/\alpha(t),\tau(t)\big)=:(\xi,\tau)$, the field of the backed-up bridge is
\begin{equation}
(\B_{r,\tilde\gamma}v')(z,t)=r-(1-\tilde\gamma)\,\xi+g(t)\,v'(\xi,\tau),
\label{eq:operator-full}
\end{equation}
which for canonical $v'$ is $\mathcal C^\pi$ on that branch.

\emph{(iii)} In readout coordinates the operator is the affine Bellman map, and inherits its contraction: with $d_q(v,w):=\sup_{z,\,t<1}\norm{q_v-q_w}=\sup_{z,t}(1-t)\norm{v-w}$,
\begin{equation}
q_{\B_{r,\tilde\gamma}v}(z,t)=r+\tilde\gamma\,q_v\big(\Phi_{r,\tilde\gamma}(z,t)\big),
\qquad
d_q(\B_{r,\tilde\gamma}v,\B_{r,\tilde\gamma}w)=\tilde\gamma\,d_q(v,w)
\label{eq:linearise}
\end{equation}
the second whenever $d_q(v,w)<\infty$, as an exact equality rather than a bound.

\end{theorem*}

The full statements below use $\F=\sigma(X_t)$ and $\HH=\sigma(X_t,H)$. Fixing $c,t$,
\begin{equation}
A_t=R-(1-\tilde\gamma)\xi_t,\quad
D=v_{\rho_1}^\star(\tau,\xi_t\mid H),\quad
w=\E[Y\mid\HH]=A_t+gD.
\label{eq:rb}
\end{equation}
Here $A_t$ is the part of the displacement the transition determines, $D$ the matched oracle of Eq.~\eqref{eq:oracle} at the retimed point, and $w$ the ideal target given everything the backup reveals. Part (i) below requires no Gaussian returns, no independent source--endpoint pair, and no deterministic rewards.

\begin{theorem*}[Conditional centering and realisation; full statement]{thm:centering}
\label{thm:centering-full}
\emph{(i) Centering.} Suppose the teacher is Eq.~\eqref{eq:oracle}, $a_t$ is $\HH$-measurable, and the target is square-integrable. For every nonterminal backup,
\begin{equation}
\E[C_{\tau(t)}\mid\HH]=0,\qquad
\E[u_a\mid\F]=\E[Y\mid\F]=v_t^\star(X_t\mid c).
\label{eq:centered}
\end{equation}
Thus every such gain has the same population regression target. Independence is a sufficient way to match an independent-CFM teacher, but it is not necessary.

\emph{(ii) Realisation.} Assume (A5) and independent fresh generation (A6): $X_1'\sim\mathcal E(v)_{S',A'}$, independent of $(X_0,R,\tilde\gamma,t)$ given $(S',A')$, as at $\rho_1=0$ in Eq.~\eqref{eq:coupling}; and let the teacher be the critic's own field $v$. Then
\begin{equation}
\E[u_\kappa\mid X_t,t,c]=(\mathcal C^\pi v)_c(X_t,t)+\E\big[\kappa g\,(v-\mathcal L\mathcal E v)_{S',A'}(\xi_t,\tau)\bigm|X_t,t,c\big],
\label{eq:realise-full}
\end{equation}
so on canonical fields the population regression target is $\mathcal C^\pi v$ for every admissible $\kappa$, and the expected squared loss is $\E\norm{f-\mathcal C^\pi v}^2$ plus a term that does not depend on the regressor $f$.

\emph{(iii) Fixed point and contraction.} On canonical fields the self-consistent iteration is therefore $\mathcal C^\pi$-iteration under any admissible gains, with fixed point $v^\pi$ and rate $\beta_p$.
\end{theorem*}

\begin{proposition*}[Exact local risk and teacher-error bias; full statement]{prop:risk}
\label{prop:risk-full}
Write the frozen teacher as $\bar v(\tau,\xi_t\mid H)=D+e$, so that $e$, measurable under $\HH$, is exactly its error against the matched oracle, and let $S=\E[\norm{(X_1'-X_0)-D}^2\mid\HH]$ be the residual bridge noise the sampled chord carries about that oracle.

\emph{(i) Local risk.} For any $\HH$-measurable $\kappa$,
\begin{align}
u_\kappa-w&=(1-\kappa)g\{(X_1'-X_0)-D\}+\kappa ge,\label{eq:error-id}\\
\E[\norm{u_\kappa-w}^2\mid\HH]&=g^2\{(1-\kappa)^2S+\kappa^2\norm{e}^2\}.
\label{eq:localrisk}
\end{align}
The oracle risk-minimizing fraction is $\kappa^\star=S/(S+\norm e^2)$ when the denominator is positive. So $\kappa$ interpolates between the two error sources by their relative size, and $\kappa=1$ is not a premise but the $e\to0$ corollary: only with an exact teacher does full correction remove all within-transition bridge noise at no cost.

\emph{(ii) Teacher error and population bias.} For a learned teacher with $\kappa ge\in L^2$, the population field is
\begin{equation}
v_\kappa(z)=v_t^\star(z)+b_t(z),\qquad
b_t(z)=\E[\kappa g e\mid X_t=z,c],\qquad
\norm{b_t}_{L^2}\leq\norm{\kappa ge}_{L^2}.
\label{eq:bias}
\end{equation}
If the teacher was calibrated under another coupling $\widetilde\rho_1$, its error separates as
\begin{equation}
e=\underbrace{\bar v-v_{\widetilde\rho_1}^\star}_{\text{teacher approximation}}
+\underbrace{v_{\widetilde\rho_1}^\star-v_{\rho_1}^\star}_{\text{coupling mismatch}},
\label{eq:teachererror}
\end{equation}
with all fields evaluated at the \emph{actual} retimed query, so the norms must be measured under that query law rather than on the teacher's training distribution. Retiming removes the intrinsic query mismatch; it removes neither term in Eq.~\eqref{eq:teachererror}.
\end{proposition*}

\begin{proposition}[Shared noise: exact under its matched teacher, which the reused critic is not]
\label{prop:leak}
\label{rem:leak}\label{thm:leak}
Let the transition be stochastic and $\rho_1=1$. (i) If the chord $X_1'-X_0$ is measurable with respect to the query and $H$, the coupling-matched teacher returns the sampled chord, so the corrected target is the sample velocity of Eq.~\eqref{eq:identity} at every gain and, under (A7), the population flow reaches the true Bellman law, at either clock. In particular, although conditioning on the current bridge point as $t\to0$ conditions on the shared source alone, so that the fitted early-time velocity is the shared-seed average of the branch maps --- in one dimension with increasing branch maps, the initial velocity toward the $W_2$ barycenter of the Bellman branch endpoint laws --- this does not fix the terminal law: branch information can re-enter through the posterior at later times. (ii) The critic the algorithm reuses is a field $v$ whose time-one flow map is the generator $\psi$. If the chord is recoverable for almost every query and the flow of $v$ is unique, $v$ equals the matched teacher almost everywhere on the query law only if the trajectories of $v$ are straight lines traversed at constant speed, which a flow-matching field's trajectories generally are not.
\end{proposition}
\noindent Its proof is in Appendix~\ref{app:pcbfkappa}.

\section{Proofs}
\label{app:proofs}

\begin{proof}[Proof of Corollary~\ref{cor:impossible}]
The identity must hold for every endpoint pair, so the coefficients of $X_0$ and $X_1'$ in $tR+\alpha(t)X'_t$ and in Eq.~\eqref{eq:current}, expanded with $X_1=R+\tilde\gamma X_1'$, match separately: $\alpha(1-t)=1-t$ and $\alpha t=\tilde\gamma t$. For $0<t<1$ they force $\alpha=1$ and $\alpha=\tilde\gamma$, incompatible for $0<\tilde\gamma<1$.
\end{proof}

\subsection{Notation}

\paragraph{Appendix shorthand.} In the dense algebra below we abbreviate the sampled chord as $C:=X_1'-X_0$. This is the \emph{chord}, not the control variate: the control variate is $C_{\tau(t)}=\bar v(\tau(t),X'_{\tau(t)}\mid H)-C$ of Eq.~\eqref{eq:cv}, which always carries a subscript. The multiplier is always written $\tilde\gamma$, and $q_v$ always denotes a readout (Table~\ref{tab:notation}).

\begin{table}[ht]
\centering
\caption{Notation and distinctions used throughout the analysis.}
\label{tab:notation}
\begin{tabularx}{\linewidth}{lX}
\toprule
Symbol&Meaning\\\midrule
$c=(s,a)$&Current state--action conditioning; fixed in local derivations.\\
$H=(R,S',A',\tilde\gamma)$&Observed backup and its actual Bellman multiplier.\\
$X_0,E,X_0'$&Bridge source, independent auxiliary seed, and successor source.\\
$X_1',C=X_1'-X_0$&Successor endpoint and successor bridge displacement (the chord).\\
$\tau(t)$&Query clock. Same-time: $\tau(t)=t$. ReBF: Eq.~\eqref{eq:clock}.\\
$X'_{\tau(t)},X_t$&Successor interpolant at the query clock, and current Bellman interpolant.\\
$\alpha,g$&Affine scale and geometric gain of the retimed clock.\\
$\kappa,a=\kappa g$&Bootstrap fraction and effective velocity coefficient.\\
$D,w$&Coupling-matched successor prediction and $\E[Y\mid X_t,H]$.\\
$e,S$&Teacher error at the actual query and conditional displacement energy.\\
$\rho_1,\widetilde\rho_1$&Training-bridge coupling and teacher-reference coupling.\\
$q_v$&Readout $z+(1-t)v(z,t)$ of a field $v$ (Definition~\ref{def:operator}).\\
$\beta_p$&Endpoint Wasserstein modulus (Proposition~\ref{prop:operator}).\\
$\beta_{\rho}(u)$&Slope of the correlated-Gaussian teacher, Eq.~\eqref{eq:gaussian}; a function of time, not $\beta_p$.\\
\bottomrule
\end{tabularx}
\end{table}

\subsection{Reusable lemmas}

\begin{lemma}[Flow matching for any coupling, including atomic endpoints]
\label{lem:flow-matching}
Let $(X_0,X_1)$ be any pair with finite first moments, $Z_t=(1-t)X_0+tX_1$, and let $v_t(z)=\E[X_1-X_0\mid Z_t=z]$ be a jointly measurable version. Then $p_t=\Law(Z_t)$ solves $\partial_tp_t+\nabla\!\cdot(p_tv_t)=0$ weakly on $[0,1)$. If the flow of $v$ is unique and this equation has a unique solution from $p_0$ in the class of (A7), the flow has marginals $p_t$. If $X_0,X_1\in L^p$, then $W_p(p_t,\Law(X_1))\to0$ as $t\uparrow1$.
\end{lemma}
\begin{proof}
For $\phi\in C_c^1(\R^d)$, $\norm{\frac{d}{dt}\phi(Z_t)}\le\sup\norm{\nabla\phi}\,\norm{X_1-X_0}\in L^1$, so dominated differentiation and the tower property give
\[
\frac{d}{dt}\E[\phi(Z_t)]=\E[\nabla\phi(Z_t)^\top(X_1-X_0)]=\E[\nabla\phi(Z_t)^\top\E[X_1-X_0\mid Z_t]]=\int\nabla\phi(z)^\top v_t(z)\,p_t(dz),
\]
which is the weak continuity equation. The marginals of the flow solve the same equation from $p_0$, so uniqueness identifies them with $p_t$. Finally $W_p(p_t,\Law(X_1))^p\le\E\norm{Z_t-X_1}^p=(1-t)^p\E\norm{X_0-X_1}^p\to0$. No density of $X_1$ is used, so atoms are covered.
\end{proof}

\begin{lemma}[Clock identities]
\label{lem:clock}
Let $0<\tilde\gamma<1$. The identity
\begin{equation}
(1-t)x_0+t(r+\tilde\gamma x_1')=tr+\alpha(t)\big\{(1-\tau(t))x_0+\tau(t)x_1'\big\}
\label{eq:clock-affine-proof}
\end{equation}
holds for every $(x_0,x_1')\in\R^d\times\R^d$ if and only if
\[
\alpha(t)=1-(1-\tilde\gamma)t,\qquad \tau(t)=\frac{\tilde\gamma t}{\alpha(t)},\qquad\text{and then}\qquad g(t):=\alpha(t)\tau'(t)=\frac{\tilde\gamma}{\alpha(t)}.
\]
Moreover $\alpha(0)=1$, $\tau(0)=0$, $\alpha(1)=\tilde\gamma$, $\tau(1)=1$, and
\begin{align*}
&t-\tau(t)=\frac{(1-\tilde\gamma)\,t(1-t)}{\alpha(t)},\qquad
\max_{0\le t\le1}\big\{t-\tau(t)\big\}=\frac{1-\sqrt{\tilde\gamma}}{1+\sqrt{\tilde\gamma}}\ \text{ at }\ t=\frac1{1+\sqrt{\tilde\gamma}},\\
&g'(t)=\frac{\tilde\gamma(1-\tilde\gamma)}{\alpha(t)^2}\ge0,\qquad \tilde\gamma\le g(t)\le1,\qquad
\frac{\tau}{1-\tau}=\tilde\gamma\,\frac{t}{1-t},\qquad g^2\,dt=\tilde\gamma\,d\tau,\\
&1-\tau=\frac{1-t}{\alpha},\qquad (1-t)\,g=\tilde\gamma\,(1-\tau).
\end{align*}
At fixed $c$ and $t\le1$, the retimed query $\xi_t=(X_t-tR)/\alpha_H(t)$ is $\HH$-measurable and $\sigma(X_t,H)=\sigma(\xi_t,H)$.
\end{lemma}
\begin{proof}
Matching the coefficients of $x_0$ and $x_1'$ in Eq.~\eqref{eq:clock-affine-proof} gives $\alpha(1-\tau)=1-t$ and $\alpha\tau=\tilde\gamma t$; adding them gives $\alpha$, and dividing gives $\tau$; conversely these values satisfy both. Then $\alpha'=-(1-\tilde\gamma)$ and $\tau'=\tilde\gamma/\alpha^2$, so $g=\alpha\tau'=\tilde\gamma/\alpha$ and $g^2dt=(\tilde\gamma^2/\alpha^2)(\alpha^2/\tilde\gamma)\,d\tau=\tilde\gamma\,d\tau$. Subtraction gives $t-\tau=\{t\alpha-\tilde\gamma t\}/\alpha=(1-\tilde\gamma)t(1-t)/\alpha$, whose derivative is $(1-\tilde\gamma)\{(1-t)^2-\tilde\gamma t^2\}/\alpha^2$; its only zero in $(0,1)$ is $1-t=\sqrt{\tilde\gamma}\,t$, i.e.\ $t=1/(1+\sqrt{\tilde\gamma})$, where $\alpha=\sqrt{\tilde\gamma}$ and the value is $(1-\tilde\gamma)\sqrt{\tilde\gamma}/\{(1+\sqrt{\tilde\gamma})^2\sqrt{\tilde\gamma}\}=(1-\sqrt{\tilde\gamma})/(1+\sqrt{\tilde\gamma})$; the endpoint values are zero, so this is the maximum, and it is $O(1-\tilde\gamma)$ as $\tilde\gamma\uparrow1$. Differentiating $g=\tilde\gamma/\alpha$ gives $g'$, and $g(0)=\tilde\gamma$, $g(1)=1$ give the bounds. The identity $1-\tau=(1-t)/\alpha$ follows from $\alpha-\tilde\gamma t=1-t$; dividing $\tau=\tilde\gamma t/\alpha$ by it gives the odds identity; and multiplying it by $g$ gives $(1-t)g=\tilde\gamma(1-\tau)$. Finally, given $(t,H)$ the maps $X_t\mapsto(X_t-tR)/\alpha_H$ and $\xi\mapsto tR+\alpha_H\xi$ are mutually inverse measurable maps because $\alpha_H(t)\ge\tilde\gamma>0$.
\end{proof}

\paragraph{Sample velocity.} At fixed sampled endpoints, $X_t=tR+\alpha(t)X'_{\tau(t)}$ with $X'_u=(1-u)X_0+uX_1'$ has the velocity of Eq.~\eqref{eq:identity}, $Y=\dot X_t=R-(1-\tilde\gamma)\xi_t+g(t)\,C$: differentiating, $\dot X_t=R+\alpha'X'_{\tau}+\alpha\tau'\,\frac{d}{du}X'_u\big|_{u=\tau}=R-(1-\tilde\gamma)\xi_t+gC$, and $\dot X_t=R+\tilde\gamma X_1'-X_0$ directly from Eq.~\eqref{eq:current}. This proves Eq.~\eqref{eq:identity} directly, without substituting a model field for a sample derivative. That substitution is justified statistically by Theorem~\ref{thm:centering}, not by the chain rule alone.

\subsection{The population operator}
\label{app:population}

\begin{proof}[Proof of Proposition~\ref{prop:operator-full}]
(i) Since $\mathcal E\mathcal L=\mathrm{id}$ by (A1), $\mathcal E\mathcal C^\pi=\mathcal E\mathcal L\T^\pi\mathcal E=\T^\pi\mathcal E$ and $\mathcal C^\pi\mathcal L=\mathcal L\T^\pi\mathcal E\mathcal L=\mathcal L\T^\pi$. If $(\mathcal C^\pi)^n=\mathcal L(\T^\pi)^n\mathcal E$ for some $n\ge1$, then $(\mathcal C^\pi)^{n+1}=\mathcal L\T^\pi(\mathcal E\mathcal L)(\T^\pi)^n\mathcal E=\mathcal L(\T^\pi)^{n+1}\mathcal E$; the identity fails at $n=0$ off canonical fields, since $v\ne\mathcal L\mathcal Ev$ in general. A fixed point $v=\mathcal C^\pi v$ lies in the range of $\mathcal L$, and applying $\mathcal E$ gives $\mathcal E(v)=\T^\pi\mathcal E(v)$, so $\mathcal E(v)=\eta^\pi$ by uniqueness of the Bellman fixed law (Banach, using the contraction below on the complete product of $p$-Wasserstein spaces), and $v=\mathcal L\T^\pi\eta^\pi=\mathcal L\eta^\pi$; conversely $\mathcal L\eta^\pi$ is fixed.

(ii) Fix $c$ and use the same draw of $H$ in two backups. Given $(S',A')=c'$, couple $X\sim\eta_{c'}$ and $\widetilde X\sim\nu_{c'}$ optimally in $W_p$, independently of $(R,\tilde\gamma)$ given $c'$; this is a coupling of $(\T^\pi\eta)_c$ and $(\T^\pi\nu)_c$ in which the reward cancels, so
\begin{align*}
W_p^p\big((\T^\pi\eta)_c,(\T^\pi\nu)_c\big)&\le\E\big[\tilde\gamma^p\norm{X-\widetilde X}^p\bigm|c\big]=\E\big[\tilde\gamma^pW_p^p(\eta_{S',A'},\nu_{S',A'})\bigm|c\big]\\
&\le\E[\tilde\gamma^p\mid c]\max_{c'}W_p^p(\eta_{c'},\nu_{c'}).
\end{align*}
Take the maximum over $c$ and the $p$th root, and apply this to $\eta=\mathcal E(v)$, $\nu=\mathcal E(w)$; since $\mathcal C^\pi v=\mathcal L\T^\pi\mathcal E(v)$ is canonical, $D_p(\mathcal C^\pi v,\mathcal C^\pi w)=\max_cW_p((\T^\pi\mathcal E v)_c,(\T^\pi\mathcal E w)_c)$. The $n$-step bound follows from (i) by induction, and $\delta_p=D_p$ on canonical fields because $\mathcal E\mathcal L=\mathrm{id}$. With one condition, $R=0$, $\tilde\gamma\equiv\bar\gamma$ and two point masses, $W_p$ is multiplied by exactly $\bar\gamma$.

The short form, Eq.~\eqref{eq:contract}, is (i)--(ii) read in $\delta_p$: part (i) gives $\mathcal E\mathcal C^\pi=\T^\pi\mathcal E$ and the unique fixed point, and since $\mathcal C^\pi v$ and $\mathcal C^\pi w$ are canonical, $\delta_p(\mathcal C^\pi v,\mathcal C^\pi w)=D_p(\mathcal C^\pi v,\mathcal C^\pi w)$, so part (ii) gives Eq.~\eqref{eq:contract}. The bound $\beta_p\le\bar\gamma$ holds by (A1), and $\beta_p\le\gamma$ whenever every multiplier is at most the nominal discount, as for $\tilde\gamma\in\{0,\gamma\}$.
\end{proof}

\subsection{Conditional centering and realisation}

\begin{proof}[Proof of Theorem~\ref{thm:centering-full}(i)]
Fix $c,t$ and a nonterminal backup. By Lemma~\ref{lem:clock}, $\tau_H(t)$ is a function of $H$ and $\sigma(X_t,H)=\sigma(X'_{\tau_H},H)$ with $X'_{\tau_H}=\xi_t$. Given $H=h$ the clock is the constant $\tau_h$, and the version of Eq.~\eqref{eq:oracle} fixed in (A5) is, for each $h$, a conditional expectation under the conditional law of $(X_0,X_1')$ given $H=h$; evaluating it at the $H$-measurable clock therefore gives
\[
D=v^\star_{\rho_1}(\tau_H,\xi_t\mid H)=\E[C\mid X'_{\tau_H},H]=\E[C\mid\HH].
\]
With this teacher $C_{\tau(t)}=D-C$, so $\E[C_{\tau(t)}\mid\HH]=0$. The gain $a_t$ is $\HH$-measurable, and $a_t(D-C)=u_a-Y$ is integrable because the target is square-integrable and $Y$ is integrable by (A1) (under (A3) $a_t=\kappa g$ is simply bounded by one), so, $D-C$ being integrable as well, pulling out the $\HH$-measurable factor \citep{durrett2019} gives $\E[u_a-Y\mid\HH]=\E[a_t(D-C)\mid\HH]=a_t\E[D-C\mid\HH]=0$, and the tower property from $\HH$ to $\F$ gives $\E[u_a\mid\F]=\E[Y\mid\F]=v^\star_t(X_t\mid c)$, the flow-matching field of the current bridge. Terminal backups use $Y=R-X_0$ with zero correction (A9), for which $u_a=Y$, so the identity holds for the whole mixture of backups. With a random training time, (A5) makes conditioning on $T=t$ reproduce the fixed-$t$ law, so the same identities hold with $\sigma(c,T,X_T)$ and $\sigma(c,T,X_T,H)$. No Gaussian returns, independent source--endpoint pair or deterministic rewards are used. Independence is sufficient for the independent-CFM teacher to be the matched one (Lemma~\ref{lem:selfconsistency}) but not necessary for centering, which holds for every coupling with its own matched teacher. For a dependent example take $X_1'=2X_0$ in one dimension: at the query $X'_u=(1+u)X_0$ the chord $C=X_0$ is recoverable, so the matched oracle equals it and centering is exact despite perfect dependence.
\end{proof}

\begin{lemma}[At independent generation the critic is the matched teacher]
\label{lem:selfconsistency}
Assume (A5) and (A6) at every backup. For every $(S',A')$, the population flow-matching field of $\eta_{S',A'}=\mathcal E(v)_{S',A'}$ under an independent source --- the field the critic at $(S',A')$ is trained to represent --- equals the matched oracle $v^\star_{\rho_1}(\tau,\xi\mid H)$ of Eq.~\eqref{eq:oracle} for every branch $H$ whose successor is $(S',A')$, for every return law and every transition kernel.
\end{lemma}
\begin{proof}
Write $G=(S',A')$ and $W=(R,\tilde\gamma)$. By (A6), $\Pr(X_1'\in B\mid X_0,W,G)=\Pr(X_1'\in B\mid G)=\mathcal E(v)_G(B)$, and by (A5) $X_0$ is independent of $H=(W,G)$. Hence
\begin{align*}
\Pr(X_0\in A,\ X_1'\in B\mid H)&=\E\big[\mathbf 1\{X_0\in A\}\Pr(X_1'\in B\mid X_0,W,G)\bigm|H\big]\\
&=\mathcal N(0,\Id_d)(A)\,\mathcal E(v)_G(B):
\end{align*}
given $H$, the pair $(X_0,X_1')$ is independent with the same law for every branch ending at $G$. Conditioning its chord on $(1-u)X_0+uX_1'$ therefore defines, for almost every query, the same regular conditional expectation as the successor critic's own independent training bridge, which is $(\mathcal L\eta)_G$.
\end{proof}

\begin{proof}[Proof of Theorem~\ref{thm:centering-full}(ii)--(iii)]
Under (A5) and (A6), Lemma~\ref{lem:selfconsistency} gives $\E[C\mid X_t,t,c,H]=(\mathcal L\mathcal Ev)_{S',A'}(\xi_t,\tau)$, the flow-matching field of the successor's generated law at the retimed query. Since $\kappa$ is $\HH$-measurable,
\[
\E[u_\kappa\mid X_t,t,c,H]=R-(1-\tilde\gamma)\xi_t+g\big\{(1-\kappa)(\mathcal L\mathcal Ev)+\kappa v\big\}_{S',A'}(\xi_t,\tau).
\]
Replacing the last $v$ by $\mathcal L\mathcal Ev$ leaves $\E[Y\mid X_t,t,c,H]$ (Eq.~\eqref{eq:identity}). By the proof of Lemma~\ref{lem:selfconsistency}, $X_0$ is independent of $(H,X_1')$, hence of the endpoint $R+\tilde\gamma X_1'$, whose law given $c$ is $(\T^\pi\mathcal E(v))_c$; so $\E[Y\mid X_t,t,c]$ is the independent-source flow-matching field of that law, namely $(\mathcal C^\pi v)_c$. What remains is $\E[\kappa g(v-\mathcal L\mathcal Ev)_{S',A'}(\xi_t,\tau)\mid X_t,t,c]$, which is Eq.~\eqref{eq:realise-full}. For square-integrable $u_\kappa$ (A4), orthogonality of conditional expectation, $\E\norm{f-U}^2=\E\norm{f-M}^2+\E\norm{U-M}^2$ for $U\in L^2$, $M=\E[U\mid\mathcal G]$ and square-integrable $\mathcal G$-measurable $f$ \citep{durrett2019}, applied with $U=u_\kappa$ and $\mathcal G=\sigma(X_t,t,c)$, gives the loss decomposition. If $v$ is canonical the defect vanishes, and $\mathcal C^\pi$ maps every admissible field to a canonical one, so the population iterates follow $\mathcal C^\pi$ and Proposition~\ref{prop:operator-full} gives the fixed point $v^\pi$ and the rate $\beta_p$.
\end{proof}

\subsection{Bellman conjugacy}

\begin{proof}[Proof of Theorem~\ref{thm:conjugacy-full}]
(i) Under (A8), let $z=tr+\varphi(t)\xi$ on the interval where $\varphi\ne0$. Differentiating the lifted path $tr+\varphi(t)X'_{\vartheta(t)}$ gives the lifted field $(\B^{\varphi,\vartheta}v)(z,t)=r+\varphi'\xi+\varphi\vartheta'v(\xi,\vartheta)$, whose readout is
\[
q_{\B^{\varphi,\vartheta}v}(z,t)=r+\{\varphi+(1-t)\varphi'\}\xi+(1-t)\varphi\vartheta'\,v(\xi,\vartheta),
\]
while the required right side is $r+\tilde\gamma q_v(\xi,\vartheta)=r+\tilde\gamma\xi+\tilde\gamma(1-\vartheta)v(\xi,\vartheta)$. Since equality must hold for every $\xi$ and every field $v$, the coefficients match separately:
\begin{equation}
\varphi+(1-t)\varphi'=\tilde\gamma,
\qquad
(1-t)\,\varphi\,\vartheta'=\tilde\gamma\,(1-\vartheta)
\label{eq:gauge}
\end{equation}
for almost every $t$. The first is linear with absolutely continuous solutions $\varphi=\tilde\gamma+A(1-t)$, and $\varphi(0)=1$ gives $A=1-\tilde\gamma$, i.e.\ $\varphi=\alpha\ge\tilde\gamma>0$; in particular $\varphi$ cannot vanish, since on the largest interval $[0,t^\ast)$ where it is nonzero it equals $\alpha$, which is positive at $t^\ast$. The second then reads $\vartheta'/(1-\vartheta)=\tilde\gamma/[(1-t)\alpha]$, and since $\frac{d}{dt}\log\frac{1-t}{\alpha}=-\tilde\gamma/[(1-t)\alpha]$, the unique absolutely continuous solution with $\vartheta(0)=0$ satisfies $1-\vartheta=(1-t)/\alpha>0$, i.e.\ $\vartheta=\tilde\gamma t/\alpha=\tau<1$. Conversely $\alpha-(1-t)(1-\tilde\gamma)=\tilde\gamma$ and $(1-t)g=\tilde\gamma(1-\tau)$ (Lemma~\ref{lem:clock}) verify both equations. If $\vartheta(t)=t$ at an interior time, then $\tau(t)=t$, and $t-\tau(t)=(1-\tilde\gamma)t(1-t)/\alpha(t)$ (Lemma~\ref{lem:clock}) vanishes only if $\tilde\gamma=1$: a same-time query forces $\tilde\gamma=1$.

(ii) For $(\alpha,\tau)$ the lifted field is $r-(1-\tilde\gamma)\xi+g\,v'(\xi,\tau)$, Eq.~\eqref{eq:operator-full}. If $v'=\mathcal L\eta'$ is canonical, the canonical readout of Definition~\ref{def:operator} and Lemma~\ref{lem:clock} give, with $X'\sim\eta'$ independent of $X_0$ and $X'_\tau=(1-\tau)X_0+\tau X'$,
\[
q_{\B v'}(z,t)=r+\tilde\gamma\,\E[X'\mid X'_\tau=\xi]=\E\big[r+\tilde\gamma X'\bigm|(1-t)X_0+t(r+\tilde\gamma X')=z\big],
\]
because $(1-t)X_0+t(r+\tilde\gamma X')=tr+\alpha X'_\tau$. This is the canonical readout of $\Law(r+\tilde\gamma X')$, and $v\mapsto q_v$ is injective for $t<1$, so $\B v'$ is the canonical field of the backed-up law: $\mathcal C^\pi$ on that branch.

(iii) The readout identity is the case $(\varphi,\vartheta)=(\alpha,\tau)$ of the display above. Subtracting it for $v$ and $w$ gives $q_{\B v}-q_{\B w}=\tilde\gamma\,(q_v-q_w)\circ\Phi_{r,\tilde\gamma}$ pointwise. The map $\Phi_{r,\tilde\gamma}$ is a bijection of $\R^d\times[0,1)$: the spatial map is affine with positive scale $\alpha$, and $\tau$ is strictly increasing from $0$ to $1$ by the odds identity. The supremum therefore transfers, and $d_q(\B v,\B w)=\tilde\gamma\,d_q(v,w)$ as an equality in $[0,\infty]$, of finite quantities whenever $d_q(v,w)<\infty$.

\end{proof}

\paragraph{Flow transfer at full gain.}
\label{app:flowtransfer}
Fix one nonterminal branch that is deterministic given $c$, with reward $r$, successor condition $c'$ and multiplier $0<\tilde\gamma<1$, and let an admissible field $v$ at $c'$, with time-one map $\psi_v$, serve as both teacher and endpoint generator. For every coupling $\rho_1$, the retimed $\kappa=1$ target equals $(\B_{r,\tilde\gamma}v)(X_t,t)$, so $\B_{r,\tilde\gamma}v$ is a version of the population field: at $\kappa=1$ the chord cancels and the target is $r-(1-\tilde\gamma)\xi_t+g\,v(\xi_t,\tau)$ with $\xi_t=(X_t-tr)/\alpha$, a function of $(X_t,t)$ because the branch is deterministic, and therefore its own conditional mean. Its flow is the affine pushforward of the flow of $v$. Let $x_u$ solve $\dot x_u=v(x_u,u)$, $x_0=x$, and put $z_t=tr+\alpha(t)x_{\tau(t)}$. Then $(z_t-tr)/\alpha=x_{\tau(t)}$ and
\[
\dot z_t=r+\alpha'x_{\tau}+\alpha\tau'\,v(x_\tau,\tau)=r-(1-\tilde\gamma)x_\tau+g\,v(x_\tau,\tau)=(\B_{r,\tilde\gamma}v)(z_t,t)
\]
by Eq.~\eqref{eq:clock} and $g=\alpha\tau'$ (Lemma~\ref{lem:clock}), so $z$ is a flow line of $\B_{r,\tilde\gamma}v$ from $z_0=x$. Conversely, a solution $z$ of that equation gives $x_u=(z_{t(u)}-t(u)r)/\alpha(t(u))$, with $t(u)$ the inverse of $\tau$, solving the equation of $v$; so uniqueness transfers. As $t\uparrow1$, $z_t\to r+\tilde\gamma x_1=r+\tilde\gamma\psi_v(x)$. The output endpoint law is therefore exactly $\Law(r+\tilde\gamma\psi_v(X_0))=(\T^\pi\mathcal E(v))_c$, and its mean is exact, however curved the flow of $v$.

\subsection{Fixed points, risk and teacher error}
\label{app:teachererr}

\begin{corollary}[Endpoint semiconjugacy and fixed-point correctness]
\label{cor:fixedpoint}
\emph{(a) Endpoint semiconjugacy for any coupling.} Let $\mathcal M^{\rm or}_{a,\chi}v$ be the population regression field at condition $c$ formed with a source--endpoint coupling $\chi$ in which $X_1'\mid H\sim\mathcal E(v)_{S',A'}$ (the coupling may correlate $X_1'$ with $X_0$), the coupling-matched teacher of $\chi$, and a gain satisfying (A3). Assume (A4), (A5), (A9), and (A7) for the output field. Then
\begin{equation}
\mathcal E\big(\mathcal M^{\rm or}_{a,\chi}v\big)_c=\big(\T^\pi\mathcal E(v)\big)_c .
\label{eq:oracle-semiconjugacy}
\end{equation}
For a non-independent coupling this is an endpoint statement only: it neither identifies the output field with $\mathcal C^\pi v$ nor gives field-level uniqueness.

\emph{(b) Fixed-point correctness.} Assume (A4), (A5), and (A7) for the actual, possibly coupled, bridge. Let the teacher be the coupling-matched oracle of Eq.~\eqref{eq:oracle}, and let the successor endpoint satisfy $X_1'\mid H\sim\eta^\pi_{S',A'}$ with $\eta^\pi=\T^\pi\eta^\pi$. At every admissible gain and every $\rho_1$, the retimed population update produces a field whose endpoint law is $\eta^\pi$: the Bellman fixed law is preserved, whatever the reward law and the transition kernel. Under independent fresh generation (A6), that field is the canonical $v^\pi=\mathcal L\eta^\pi$, which is therefore a fixed point; for a non-independent coupling it is the field of the coupled bridge, which in general differs from $v^\pi$ while sharing its endpoint law. For the same-time update the endpoint conclusion holds if $\rho_1=1$ and the chord $C$ is measurable with respect to $\sigma(X'_t,H)$ for almost every $t$, or $\kappa=0$, or $\tilde\gamma=1$; in the correlated-Gaussian model of Proposition~\ref{prop:scale} these are the only cases in which the same-time field has zero residual, the residual being Eq.~\eqref{eq:sametime-bias}. For scalar returns a generator that is the time-one map of a unique one-dimensional flow gives the measurability for every $t<1$, since that map is increasing, so $(1-t)x+t\psi(x)$ is injective; a deterministic generator does not give it in general.
\end{corollary}
\begin{proof}
(a) By Theorem~\ref{thm:centering-full}(i), with the conditioning on the training time justified by (A5), the population regression field is $m_c(z,t)=\E[X_1-X_0\mid X_t=z,t,c]$ with $X_1=R+\tilde\gamma X_1'$ (and $X_1=R$ on terminal branches): the flow-matching field of the actual current bridge under whatever joint law $\chi$ induces. Lemma~\ref{lem:flow-matching} and (A7) give $\mathcal E(m)_c=\Law(R+\tilde\gamma X_1'\mid c)$, and since $X_1'\mid H\sim\mathcal E(v)_{S',A'}$ this is $(\T^\pi\mathcal E(v))_c$. A coupled bridge field need not be canonical, so nothing equates the output fields themselves.

(b) Part (a), applied with $\eta^\pi$ in the role of $\mathcal E(v)$, gives the endpoint law $(\T^\pi\eta^\pi)_c=\eta^\pi_c$; this needs $\eta^\pi=\T^\pi\eta^\pi$, since centering alone identifies only the one-step backed-up field. Under (A6) the matched oracle is the independent-source field of $\eta^\pi$ (Lemma~\ref{lem:selfconsistency}), so Theorem~\ref{thm:centering-full}(ii) with the canonical input $v^\pi$ identifies the field with $\mathcal L\T^\pi\eta^\pi=\mathcal L\eta^\pi$. Coupled field equality fails in general: take one condition, $\tilde\gamma=\tfrac12$, $R\sim\mathcal N(0,\tfrac34)$ independent of $X_0$, and $\eta^\pi=\mathcal N(0,1)$. The canonical flow of $\mathcal N(0,1)$ from an independent $\mathcal N(0,1)$ source has time-one map $\psi(x)=x$, so $\rho_1=1$ gives $X_1'=X_0$, a recoverable chord, and $B=R+\tfrac12X_0\sim\mathcal N(0,1)$ with $\Cov(X_0,B)=\tfrac12$. A standard Gaussian bridge whose source--endpoint correlation is $\varrho$ has field slope $\{(1-t)(\varrho-1)+t(1-\varrho)\}/\{(1-t)^2+t^2+2\varrho t(1-t)\}$; at $t=\tfrac14$ this is $-\tfrac4{13}$ for $\varrho=\tfrac12$ and $-\tfrac45$ for $\varrho=0$: the same endpoint law and a different field.

Same-time sufficiency is the argument of Proposition~\ref{prop:leak}(i), which uses neither the clock nor the randomness of the transition: if the chord is $\sigma(X'_t,H)$-measurable, the same-time oracle returns it exactly, the correction vanishes, the target is the sampled velocity, and Lemma~\ref{lem:flow-matching} with (A7) gives the Bellman endpoint; $\kappa=0$ is ordinary flow matching, and at $\tilde\gamma=1$ the two clocks agree. In the Gaussian model the residual Eq.~\eqref{eq:sametime-bias} is a product of $t(1-t)$, $1-\rho_1^2$ and $1-\tilde\gamma$ over positive variances, so for $0<t<1$ these are its only zero cases (together with $\kappa=0$); a nonzero field residual does not by itself preclude a correct endpoint: at full retimed gain on one deterministic branch the endpoint law is $\Law(r+\tilde\gamma\psi_v(X_0))$ for every admissible teacher field $v$, which is correct whenever $v$ generates the correct successor law (Appendix~\ref{app:flowtransfer}). For the scalar clause, a time-one map of a unique one-dimensional flow is increasing, so $x\mapsto(1-t)x+t\psi(x)$ is strictly increasing for $t<1$, the query recovers $x$, and hence the chord $\psi(x)-x$. Determinism alone is not enough in $d\ge2$: the rotation $X_1'=-X_0$ is the time-one map of a smooth planar flow, yet at $t=\tfrac12$ it gives $X'_t=0$, so the same-time oracle returns $\E[X_1'-X_0]=0$ while $X_t=\tfrac{1-\tilde\gamma}{2}X_0+\tfrac12r$ reveals the chord $-2X_0$. That example fails only at one time, so it refutes a pointwise claim, not the almost-everywhere condition.
\end{proof}

\begin{proof}[Proof of Proposition~\ref{prop:risk-full}]
(i) Subtracting $w=A_t+gD$ from Eq.~\eqref{eq:target} gives Eq.~\eqref{eq:error-id}. For $\HH$-measurable $\kappa$, squaring and conditioning on $\HH$ produces the cross term $2\kappa(1-\kappa)g^2\inner{\E[C-D\mid\HH]}{e}=0$, because $D=\E[C\mid\HH]$ (proof of Theorem~\ref{thm:centering-full}(i)), and the remaining terms give Eq.~\eqref{eq:localrisk}. The derivative of $(1-\kappa)^2S+\kappa^2\norm e^2$ is $2\{-S+\kappa(S+\norm e^2)\}$, so $\kappa^\star=S/(S+\norm e^2)\in[0,1]$ is the unique minimiser when the denominator is positive; if both terms vanish, every fraction has zero conditional risk.

(ii) By Eq.~\eqref{eq:error-id} and the tower property, $\E[u_\kappa\mid X_t]-\E[Y\mid X_t]=\E[\kappa ge\mid X_t]$, and conditional Jensen, $\E\norm{\E[U\mid\mathcal G]}^2\le\E\norm U^2$ for $U\in L^2$ \citep{durrett2019}, gives the $L^2$ bound. Eq.~\eqref{eq:teachererror} adds and subtracts the oracle at $\widetilde\rho_1$. With a constant $\kappa\in[0,1]$ and $g\le1$, the triangle inequality gives
\[
\norm{b_t}_{L^2}\leq\kappa\left(\norm{\bar v-v_{\widetilde\rho_1}^\star}_{L^2(\text{actual queries})}
+\norm{v_{\widetilde\rho_1}^\star-v_{\rho_1}^\star}_{L^2(\text{actual queries})}\right).
\]
A field approximation error of a finite regression procedure must be added separately. If the endpoint solver changes the law of $X_1'$, the oracle here belongs to that approximate law, and Bellman error relative to the exact successor distribution is an additional term.
\end{proof}

\subsection{Contraction moduli}
\label{app:moduli}
The branchwise contraction depends on the metric. In the readout metric $d_q$ of Theorem~\ref{thm:conjugacy-full}(iii) the modulus is exactly $\tilde\gamma$; in the velocity supremum norm it is exactly $1$.

\subsection{Correlated-Gaussian identities and the terminal scale error}
\label{app:gaussian}

\begin{proposition}[Correlated-Gaussian residuals and the terminal scale error of the constant-gain same-time target]
\label{prop:scale}
Let $(X_0,V)$ be centred standard Gaussians with correlation $\rho_1$, $X_1'=\mu+\sigma V$ with $\sigma>0$, and fix $R=r$ and $0<\tilde\gamma\le1$. Put $P_u=(1-u)^2+u^2\sigma^2+2u(1-u)\rho_1\sigma$ and $Q_t=\Var(X_t)=(1-t)^2+\tilde\gamma^2t^2\sigma^2+2(1-t)\tilde\gamma t\rho_1\sigma$.

\emph{(i) Matched teacher and residuals.} The matched teacher queried at clock $u$, in particular at the retimed clock $u=\tau(t)$, is
\begin{equation}
v_{\rho_1}^\star(u,b)=\mu+\beta_{\rho_1}(u)(b-u\mu),\qquad
\beta_{\rho_1}(u)=\frac{(1-u)(\rho_1\sigma-1)+u(\sigma^2-\rho_1\sigma)}
{(1-u)^2+u^2\sigma^2+2u(1-u)\rho_1\sigma},
\label{eq:gaussian}
\end{equation}
and the remaining bridge uncertainty is
\begin{equation}
S_t=\Var(C\mid X_t,H)=\frac{\sigma^2\alpha(t)^2(1-\rho_1^2)}{Q_t}.
\label{eq:gaussian-local}
\end{equation}
For a teacher calibrated at $\widetilde\rho_1$ and queried at clock $u=u(t)$,
\begin{equation}
\E[v_{\widetilde\rho_1}^\star(u,X'_u)-C\mid X_t=z]
=\frac{\beta_{\widetilde\rho_1}(u)K_{BZ}(u)-K_{CZ}}{Q_t}
\{z-t(r+\tilde\gamma\mu)\},
\label{eq:general-gaussian-bias}
\end{equation}
where $K_{CZ}=(1-t)(\rho_1\sigma-1)+\tilde\gamma t(\sigma^2-\rho_1\sigma)$ and $K_{BZ}(u)=(1-u)(1-t)+u\tilde\gamma t\sigma^2+\rho_1\sigma\{(1-u)\tilde\gamma t+u(1-t)\}$. The numerator vanishes at the matched retimed clock, and at the matched same-time clock it equals $t(1-t)\sigma^2(1-\rho_1^2)(1-\tilde\gamma)/P_t$, so that
\begin{equation}
\E[v_{\rho_1}^\star(t,X'_t)-C\mid X_t=z]
=\frac{t(1-t)\sigma^2(1-\rho_1^2)(1-\tilde\gamma)}{P_tQ_t}
\{z-t(r+\tilde\gamma\mu)\}.
\label{eq:sametime-bias}
\end{equation}

\emph{(ii) Terminal scale error of the constant-gain same-time target.} In the correlated-Gaussian model above, with $0<\tilde\gamma<1$, $\rho_1=0$, an exact matched teacher and constant gain $\tilde\gamma$,
\begin{equation}
\ln\frac{\sigma_{\rm out}}{\tilde\gamma\sigma}=\frac{-\,\tilde\gamma\ln\tilde\gamma}{1+\tilde\gamma}\qquad\text{for every }\sigma>0,
\label{eq:scalefree}
\end{equation}
independently of the return scale. The error is positive: the learned law is over-dispersed. The retimed target, and the same-time target at $\rho_1=1$, both give exactly zero.

\end{proposition}
\begin{proof}
(i) Centre the chord, the successor query and the current point: $\bar C=\sigma V-X_0$, $\bar B_u=(1-u)X_0+u\sigma V$, $\bar X_t=(1-t)X_0+\tilde\gamma t\sigma V=z-t(r+\tilde\gamma\mu)$. Direct expansion gives $\Cov(\bar C,\bar B_u)=(1-u)(\rho_1\sigma-1)+u(\sigma^2-\rho_1\sigma)$, $\Var(\bar B_u)=P_u$, $\Cov(\bar C,\bar X_t)=K_{CZ}$, $\Cov(\bar B_u,\bar X_t)=K_{BZ}(u)$ and $\Var(\bar X_t)=Q_t$. The Gaussian regression formula $\E[U\mid B]=\E U+\Cov(U,B)\Var(B)^{-1}(B-\E B)$ gives Eq.~\eqref{eq:gaussian}. The map $(X_0,V)\mapsto(\bar C,\bar X_t)$ has determinant $-\sigma\{1-(1-\tilde\gamma)t\}=-\sigma\alpha$ and $\det\Cov(X_0,V)=1-\rho_1^2$, so $\Var(\bar C\mid\bar X_t)=\det\Cov(\bar C,\bar X_t)/\Var(\bar X_t)=\sigma^2\alpha^2(1-\rho_1^2)/Q_t$, which is Eq.~\eqref{eq:gaussian-local}; at $\rho_1=1$ it vanishes. The teacher residual is $v^\star_{\widetilde\rho_1}(u,X'_u)-C=\beta_{\widetilde\rho_1}(u)\bar B_u-\bar C$; regressing it on $\bar X_t$ gives Eq.~\eqref{eq:general-gaussian-bias}.

At the retimed clock $u=\tau$, $1-\tau=(1-t)/\alpha$ and $\tau=\tilde\gamma t/\alpha$ (Lemma~\ref{lem:clock}; at $\tilde\gamma=1$ they hold directly, with $\alpha\equiv1$ and $\tau(t)=t$) give $\bar B_\tau=\bar X_t/\alpha$, so conditioning on $X_t$ is conditioning on the query, and with $\widetilde\rho_1=\rho_1$ the residual has zero conditional mean by the tower property; in coefficients, $\Cov(\bar C,\bar B_\tau)=K_{CZ}/\alpha$, $P_\tau=Q_t/\alpha^2$ and $K_{BZ}(\tau)=Q_t/\alpha$, so $\beta_{\rho_1}(\tau)K_{BZ}(\tau)-K_{CZ}=(\alpha K_{CZ}/Q_t)(Q_t/\alpha)-K_{CZ}=0$. At the same-time clock $u=t$ with $\widetilde\rho_1=\rho_1$, write $k=\Cov(\bar C,\bar B_t)$; then
\begin{align*}
k&=K_{CZ}+\Delta_1,&\Delta_1&=(1-\tilde\gamma)t(\sigma^2-\rho_1\sigma),\\
K_{BZ}(t)&=P_t-\Delta_2,&\Delta_2&=(1-\tilde\gamma)t\{t\sigma^2+(1-t)\rho_1\sigma\},
\end{align*}
so $P_t\{\beta_{\rho_1}(t)K_{BZ}(t)-K_{CZ}\}=kK_{BZ}(t)-K_{CZ}P_t=\Delta_1P_t-k\Delta_2=(1-\tilde\gamma)t\,\Pi$ with
\begin{align*}
\Pi&=(\sigma^2-\rho_1\sigma)P_t-k\{t\sigma^2+(1-t)\rho_1\sigma\}\\
&=(1-t)^2\sigma^2(1-\rho_1^2)+t(1-t)\sigma^2(1-\rho_1^2)=(1-t)\sigma^2(1-\rho_1^2),
\end{align*}
where the $t^2$ terms cancel and the $(1-t)^2$ and $t(1-t)$ coefficients are collected in powers of $\rho_1\sigma$. Dividing by $P_tQ_t$ gives Eq.~\eqref{eq:sametime-bias}.

(ii) Write $\tilde\gamma$ for the multiplier, $P_t=(1-t)^2+t^2\sigma^2$ and $Q_t=(1-t)^2+\tilde\gamma^2t^2\sigma^2$ (the $\rho_1=0$ values). The exact field of the current bridge is linear in the centred coordinate with slope $Q_t'/(2Q_t)$, because $\Cov(\dot X_t,X_t)=\tfrac12\frac{d}{dt}\Var(X_t)$, and its flow has variance $Q_t$, ending at $\tilde\gamma^2\sigma^2$. Under constant same-time gain $\tilde\gamma$, the population field adds $\tilde\gamma$ times the residual of Eq.~\eqref{eq:sametime-bias} at $\rho_1=0$, an excess slope
\[
\Delta a(t)=\frac{\tilde\gamma(1-\tilde\gamma)\,t(1-t)\sigma^2}{P_tQ_t},
\]
and leaves the mean path unchanged. A centred linear flow multiplies its standard deviation by $\exp(\int_0^1a)$, so $\ln(\sigma_{\rm out}/\tilde\gamma\sigma)=\int_0^1\Delta a\,dt$. With $w=t/(1-t)$, $dt=dw/(1+w)^2$, $t(1-t)=w/(1+w)^2$, $P_t=(1+\sigma^2w^2)/(1+w)^2$ and $Q_t=(1+\tilde\gamma^2\sigma^2w^2)/(1+w)^2$, and then $s=\sigma w$,
\begin{align*}
\int_0^1\Delta a\,dt&=\tilde\gamma(1-\tilde\gamma)\int_0^\infty\frac{s\,ds}{(1+s^2)(1+\tilde\gamma^2s^2)}\\
&=\tilde\gamma(1-\tilde\gamma)\lim_{M\to\infty}\frac{1}{2(1-\tilde\gamma^2)}\Big[\ln\frac{1+s^2}{1+\tilde\gamma^2s^2}\Big]_0^M
=\frac{\tilde\gamma(1-\tilde\gamma)(-\ln\tilde\gamma)}{1-\tilde\gamma^2},
\end{align*}
which is $-\tilde\gamma\ln\tilde\gamma/(1+\tilde\gamma)$, independent of $\sigma$ and positive for $0<\tilde\gamma<1$. The retimed target has zero residual (part (i)), and the same-time residual carries the factor $1-\rho_1^2$, which vanishes at $\rho_1=1$.

\end{proof}

\subsection{Shared noise}
\label{app:pcbfkappa}

\begin{proof}[Proof of Proposition~\ref{prop:leak}]
(i) If the chord is measurable with respect to the query and $H$, the matched oracle is $D=\E[C\mid\text{query},H]=C$, so the control variate vanishes and $u_a=Y=R+\tilde\gamma X_1'-X_0$ at every gain; the regression is ordinary flow matching onto the Bellman endpoint, and Lemma~\ref{lem:flow-matching} with (A7) identifies the generated law. The residual is pointwise zero, so this holds at either clock. For the early-time remark take $d=1$, reward zero, a branch sign $S\in\{-1,1\}$ with equal probability independent of $X_0\sim\mathcal N(0,1)$, and shared-noise generators $X_1'=Sa+\sigma X_0$ with $a,\sigma>0$. Since $X'_u=uSa+(1-u+u\sigma)X_0$ with $1-u+u\sigma>0$, the chord is measurable with respect to $(X'_u,S)$, so the flow reaches $\tfrac12\mathcal N(-\tilde\gamma a,\tilde\gamma^2\sigma^2)+\tfrac12\mathcal N(\tilde\gamma a,\tilde\gamma^2\sigma^2)$, of variance $\tilde\gamma^2(\sigma^2+a^2)$. At $t=0$, however, $X_t=X_0$ carries no branch information and the fitted velocity is $\E[\tilde\gamma Sa+\tilde\gamma\sigma z-z]=(\tilde\gamma\sigma-1)z$: the velocity of the increasing map toward $\mathcal N(0,\tilde\gamma^2\sigma^2)$, the one-dimensional $W_2$ barycenter of the two Bellman branch endpoint laws $\mathcal N(\pm\tilde\gamma a,\tilde\gamma^2\sigma^2)$, whose variance is only $\tilde\gamma^2\sigma^2$. Branch information re-enters through the posterior at later times.

(ii) Write $b_u(x_0)=(1-u)x_0+u\,\psi(x_0)$ for the query built from a source $x_0$, where $\psi$ is the time-one map of the reused field $v$. Recoverability makes the matched teacher at $(u,b_u(x_0))$ equal to $\psi(x_0)-x_0$ for almost every query. The query law is the image of (clock law)$\,\otimes\,\mathcal N(0,\Id_d)$ under $(u,x_0)\mapsto(u,b_u(x_0))$, and the clock law has a positive density on $(0,1)$. If $v$ equals the matched teacher almost everywhere on the query law, Fubini's theorem gives, for almost every $x_0$, $v(b_u(x_0),u)=\psi(x_0)-x_0=\frac{d}{du}b_u(x_0)$ for almost every $u$; so $u\mapsto b_u(x_0)$ is an absolutely continuous solution of $x_u=x_0+\int_0^uv(x_s,s)\,ds$, and uniqueness of the flow (for absolutely continuous solutions; for continuous fields this is ordinary ODE uniqueness) makes it the trajectory of $v$, which is therefore straight at constant speed. Conversely, if the trajectory from almost every $x_0$ is $b_u(x_0)$, its velocity $\psi(x_0)-x_0$ equals $v(b_u(x_0),u)$ for almost every $u$, which is the matched teacher.

\end{proof}

\begin{figure}[h]
\centering
\includegraphics[width=\linewidth]{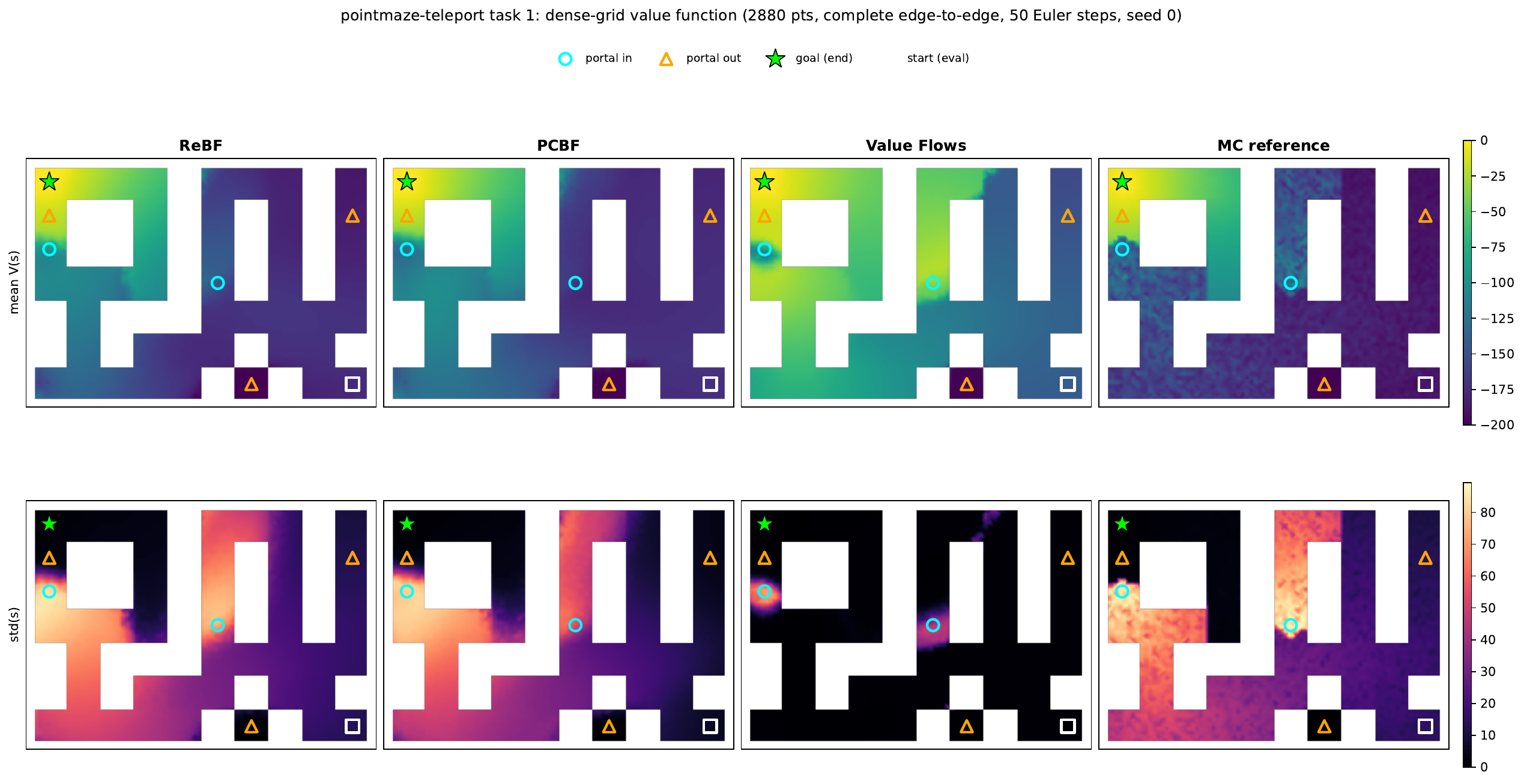}
\caption{\textbf{Learned value and uncertainty over the whole maze}, OGBench \texttt{pointmaze-teleport} (task $1$, $\gamma=0.995$), seed $0$, $50$ Euler steps. Columns: ReBF, PCBF, Value Flows (the published critic) and a Monte Carlo reference; top row, the mean return $V(s)$ of the law each critic generates; bottom row, its standard deviation. Each map is evaluated at $2{,}880$ points ($8\times8$ per free cell, spacing $0.5$) and linearly interpolated within free cells; walls are white. Markers are fixed maze geometry: portal entrances (circles), portal exits (triangles, one of them the sealed trap), the goal (star) and the ``start'' evaluation point (square). Against the reference, ReBF's mean map has mean absolute error $13.0$ (bias $+9.3$) and its standard-deviation map $3.9$, with spatial correlation $0.96$ to the reference's and average ratio $1.05$; PCBF's are $17.4$ ($+11.9$), $6.5$, $0.91$ and $0.92$; Value Flows' are $60.6$ ($+60.6$), $27.8$, $0.31$ and $0.09$, its uncertainty collapsing almost everywhere. The reference has only $30$ rollouts per point, which makes its panels grainy and puts a floor of about $4.4$ on the mean-map error and $1.7$ on the standard-deviation error; ReBF's and PCBF's errors are two to four times these floors. One seed: a picture, not an estimate.}
\label{fig:ogbench-map}
\end{figure}

\section{Toy Results}
\label{app:toysfull}

\paragraph{Against a tuned comparator, on held-out seeds.}
Table~\ref{tab:headline} tunes $\kappa$ separately for each family --- the same-time shared-noise family and the retimed independent-noise family --- selecting on one seed split and scoring on the other.

\begin{table}[t]
\centering
\caption{Tuned same-time ($\rho_1=1$) versus tuned retimed ($\rho_1=0$), $\kappa$ selected on five seeds and scored on the five held-out seeds; both split directions shown. $W_1$ against the exact law at $50$ ODE steps, $4\times10^5$ draws. Every level in this table is at least $19\times$ the measured evaluation floor for its environment; the gaps on \textsc{solitaire} and \textsc{discrete-mc} are significant in both split directions ($p\le0.04$ at $n=5$), while on \textsc{bernoulli} they are $1.3$--$2.2\times$ the floor at $n=15$, which is why we call it a tie (Appendix~\ref{app:protocol}). Scored $n$ is $15$ on \textsc{bernoulli} and $5$ elsewhere; ties are declared using the $t$ critical value at the relevant $n$.}
\label{tab:headline}
\begin{tabular}{llrrrrl}
\toprule
Env & Split & $\kappa^\star$ (same-time) & $\kappa^\star$ (retimed) & Ratio & $p$ & Verdict\\
\midrule
\textsc{bernoulli} & A$|$B & 0.1 & 1 & 1.06$\times$ & 0.455 & powered null\\
\textsc{bernoulli} & B$|$A & 0.02 & 1 & 1.10$\times$ & 0.292 & powered null\\
\textsc{solitaire} & A$|$B & 0 & 1 & 1.62$\times$ & 0.040 & retimed\\
\textsc{solitaire} & B$|$A & 0.05 & 1 & 1.89$\times$ & 0.003 & retimed\\
\textsc{discrete-mc} & A$|$B & 0.3 & 1 & 1.64$\times$ & 0.004 & retimed\\
\textsc{discrete-mc} & B$|$A & 0.3 & 1 & 1.66$\times$ & 0.0003 & retimed\\
\bottomrule
\end{tabular}
\end{table}

Against a tuned same-time comparator the retimed family wins on \textsc{solitaire} and \textsc{discrete-mc}. On \textsc{bernoulli} the fifteen-seed held-out gaps do not separate ($+0.00086\pm0.00112$, $p=0.455$; $+0.00141\pm0.00129$, $p=0.292$), at $21$--$25\times$ the evaluation floor, so the advantage is environment-dependent. \textsc{bernoulli} is also the only environment with $\gamma/(1-\gamma)=1$, and this design does not separate smoothness from the absence of amplification. Against pinned PCBF ($\kappa=1$) the gaps are larger ($0.053$ vs $0.017$, $0.680$ vs $0.091$, $0.594$ vs $0.302$). In floor units the tuned comparison is $19\times$ against $24\times$, $28\times$ against $49\times$, and $40\times$ against $66\times$.

The result the theory predicts is where each optimum sits. The retimed $\kappa^\star$ is $1$ in every environment and split, so tuning gains nothing. The same-time optimum is never at $\kappa=1$; that rung is its worst, and tuning moves it toward $\kappa=0$. Theorem~\ref{thm:centering} makes the retimed correction unbiased at every gain. At $\rho_1=1$ the same-time residual of Eq.~\eqref{eq:sametime-bias} vanishes, but the reused critic is generally not the teacher its coupling requires (Proposition~\ref{prop:leak}(ii)), and that error enters scaled by $\kappa$ (Eq.~\eqref{eq:bias}). On the full $(\rho_1,\kappa)$ surface the best retimed cell is $(\rho_1=0,\kappa=1)$, ReBF itself, in all three environments.

\paragraph{Clock and coupling interact.}
Holding $\kappa=1$ and measuring the clock gap (same-time minus retimed $W_1$) at the two coupling extremes:

\begin{center}
\begin{tabular}{lrrrl}
\toprule
Env & gap at $\rho_1=0$ & gap at $\rho_1=1$ & difference & $p$\\
\midrule
\textsc{bernoulli} & $+0.0408$ & $-0.0562$ & $-0.0970$ & $<10^{-4}$\\
\textsc{solitaire} & $+0.1623$ & $-0.0460$ & $-0.2083$ & $0.0004$\\
\textsc{discrete-mc} & $+0.2950$ & $-0.1051$ & $-0.4001$ & $<10^{-4}$\\
\bottomrule
\end{tabular}
\end{center}

\paragraph{With a reused critic, shared noise settles at a biased level.}

On \textsc{solitaire} the shared-noise $W_1$ rises with training and PCBF then flattens (Figure~\ref{fig:allenvs}c). From $2\,000$ to $50\,000$ steps, five seeds, PCBF moves $0.544\to0.751$ while ReBF falls $0.381\to0.126$; a rerun of the same seeds ends at $0.698$ and $0.123$. PCBF is about six times higher at the last checkpoint, with mean bias $-14.0\%\pm1.2$ ($n=10$). The arms stay together through step $1\,000$ and separate between $1\,000$ and $2\,000$ ($-0.166\pm0.040$), so the bias appears once the critic is accurate enough for the coupling to matter. The retimed shared-noise arm rises the same way. Table~\ref{tab:depth} shows the penalty vanishes as the backup deepens. It does not scale with the horizon: at one step the bias is $-14.0\%$, $-28.5\%$ and $-29.7\pm2.5\%$ at $\gamma=0.9$, $0.99$ and $0.995$, while $\gamma/(1-\gamma)$ runs $9\to99\to199$. A matched teacher makes shared noise exact (Proposition~\ref{prop:leak}(i)); the reused critic is that teacher only if its flow is straight (Proposition~\ref{prop:leak}(ii)). A single deterministic branch leaves the mean exact at full retimed gain (Appendix~\ref{app:flowtransfer}). We have no population model that also matches the measured size.

\begin{table}[ht]
\centering
\caption{Mean bias (\%) against the exact mean, with $W_1$ beneath, as a function of backup depth $n$; \textsc{solitaire}, $\gamma=0.9$, five seeds, $50$ ODE steps. Standard errors are $0.6$--$1.4$ percentage points throughout. An $n$-step target uses $\tilde\gamma=\gamma^n$ and the discounted $n$-step reward. $W_1$ here compares $4{,}000$ learned draws against a $20{,}000$-draw sampled reference, whose floor is $0.043\pm0.017$ (Appendix~\ref{app:protocol}); at $n=20$ ReBF and PCBF differ by less than that floor's standard deviation.}
\label{tab:depth}
\begin{tabular}{lrrrrr}
\toprule
Arm & $n=1$ & $n=3$ & $n=5$ & $n=10$ & $n=20$\\
\midrule
ReBF (retimed, $\rho_1{=}0$, $\kappa{=}1$) & $-1.64$ & $-3.05$ & $-1.68$ & $-0.25$ & $-0.42$\\
 & \footnotesize 0.097 & \footnotesize 0.155 & \footnotesize 0.123 & \footnotesize 0.119 & \footnotesize 0.134\\
PCBF (same-time, $\rho_1{=}1$, $\kappa{=}1$) & $-14.17$ & $-3.94$ & $-0.45$ & $-0.31$ & $+0.08$\\
 & \footnotesize 0.709 & \footnotesize 0.427 & \footnotesize 0.250 & \footnotesize 0.140 & \footnotesize 0.126\\
uncorrected ($\kappa{=}0$) & $-1.36$ & $-2.48$ & $-0.97$ & $-0.35$ & $+0.55$\\
 & \footnotesize 0.162 & \footnotesize 0.156 & \footnotesize 0.155 & \footnotesize 0.112 & \footnotesize 0.112\\
\bottomrule
\end{tabular}
\end{table}

\section{All Environments and All Analyses}
\label{app:allenvs}
Figure~\ref{fig:allenvs} collects all analysis of Section~\ref{sec:toys} in every environment, with four-rooms successor features (Appendix~\ref{app:env-fourrooms}) as a fourth column. The main-text Figure~\ref{fig:main} shows one environment for each row, and Figure~\ref{fig:grid} shows the full $(\rho_1,\kappa)$ surface of the toys that rows (a) and (b) cut through (on seeds $0$--$4$ only, so its corner cells differ from the ten-seed values in those rows, by up to $16\%$ on \textsc{bernoulli}). The toy columns are read out at $50$ ODE steps; the four-rooms column at $128$, from $2\times10^4$ critic draws against a $4\times10^5$-rollout reference, and its training curves use that same instrument. The test-time budget is in Appendix~\ref{app:budget}.

\emph{(a) Correction weight.} On the toys the retimed $\rho_1=0$ slice falls to its minimum at $\kappa=1$, and the same-time shared-noise family is worst there (sweep minima $\kappa=0.02$, $0$, $0.3$). Green is Value Flows on its own dial $\kappa_{\rm VF}=\lambda_{\rm DCFM}/(\lambda_{\rm BCFM}+\lambda_{\rm DCFM})$, not our $\kappa$; its $\kappa_{\rm VF}=1$ cell diverged and is omitted. On four-rooms the retimed slice falls from $0.106$ at $\kappa=0$ to $0.028$ at $\kappa=1$; the same-time family is lowest at $\kappa=0.75$ ($0.038$).
\emph{(b) Noise coupling.} At $\kappa=1$ the retimed $W_1$ rises with $\rho_1$ from its minimum at $\rho_1=0$, and the same-time curve is U-shaped; at $\kappa=0.1$ the two clocks lie within $16\%$. On four-rooms a partially coupled same-time critic ($\rho_1=0.75$) ties ReBF ($0.0289$ against $0.0300$, $p=0.24$; held-out seeds $5$--$14$: $0.0277$ against $0.0273$, $p=0.75$), while PCBF at $\rho_1=1$ is higher on all $15$ seeds.
\emph{(c) Training.} On \textsc{solitaire} the arms coincide through step $1\,000$; the shared-noise arms at $\rho_1=1$, $\kappa=1$ then rise, so the effect tracks the coupling. On four-rooms the three arms coincide through $5\,000$ steps and at $20\,000$ read $0.028$, $0.041$ and $0.052$ for ReBF, PCBF and the uncorrected critic.

\begin{figure}[p]
\centering
\includegraphics[width=\linewidth]{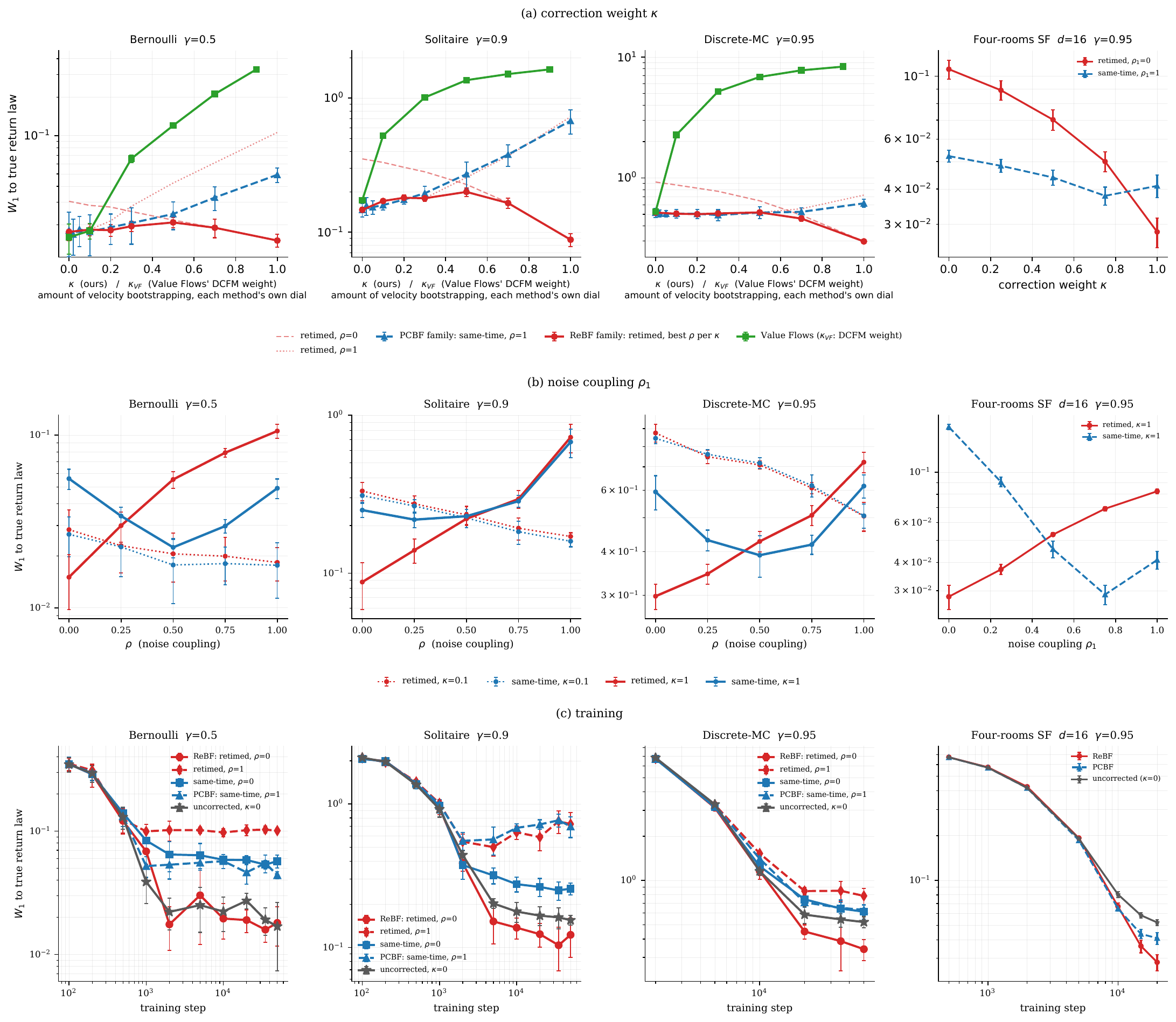}
\caption{\textbf{Every analysis in every environment} (columns \textsc{bernoulli}, \textsc{solitaire}, \textsc{discrete-mc} and four-rooms successor features, $d=16$): (a) correction weight, (b) noise coupling, (c) training, described in Appendix~\ref{app:allenvs}. $W_1$ to the exact law ($4\times10^5$ draws) in the toy panels of rows (a) and (b), and to a $4\times10^5$-rollout reference in every four-rooms panel; scoring and floors in Appendix~\ref{app:protocol}. Error bars are $\pm1$ s.d., except $\pm1$ s.e.m.\ on the solid red and green curves of the toy panels of row (a). The $\kappa=0$ arm is the uncorrected critic, shown for context. The toy training curves and Value Flows (green, row a) are scored on a coarser instrument: $4\times10^3$ draws against a sampled $2\times10^4$-draw reference, with floors of about $0.011$, $0.050$ and $0.075$.}
\label{fig:allenvs}
\end{figure}

\emph{Four-rooms maps.} Figure~\ref{fig:fourrooms} maps one seed's mean and dispersion errors over all $104$ cells of four-rooms and draws its joint law at one evaluation state. Away from $w=M^\top e_{15}$, the same-time critics' dispersion errors have a sign. Over the other $10$ directions at the four evaluation states, PCBF and the uncorrected critic BCFM are under-dispersed on $15$ of $15$ seeds (mean sd ratio $0.922\pm0.003$ and $0.908\pm0.002$, mean $\pm$ s.e.), and more so along $w=M^\top e_0$, the slowest room mode (horizon $18.4$; $0.839$ and $0.748$). ReBF is at $1.000\pm0.002$, below $1$ on $7$ of $15$ seeds, and Value Flows is over-dispersed ($1.067\pm0.003$, above $1$ on $15$ of $15$). With $w=M^\top e_{15}$ included, as in Table~\ref{tab:compare}, PCBF stays below $1$ ($0.976$) but BCFM and ReBF move above it ($1.077$ and $1.051$).

\begin{figure}[t]
\centering
\includegraphics[width=0.85\linewidth]{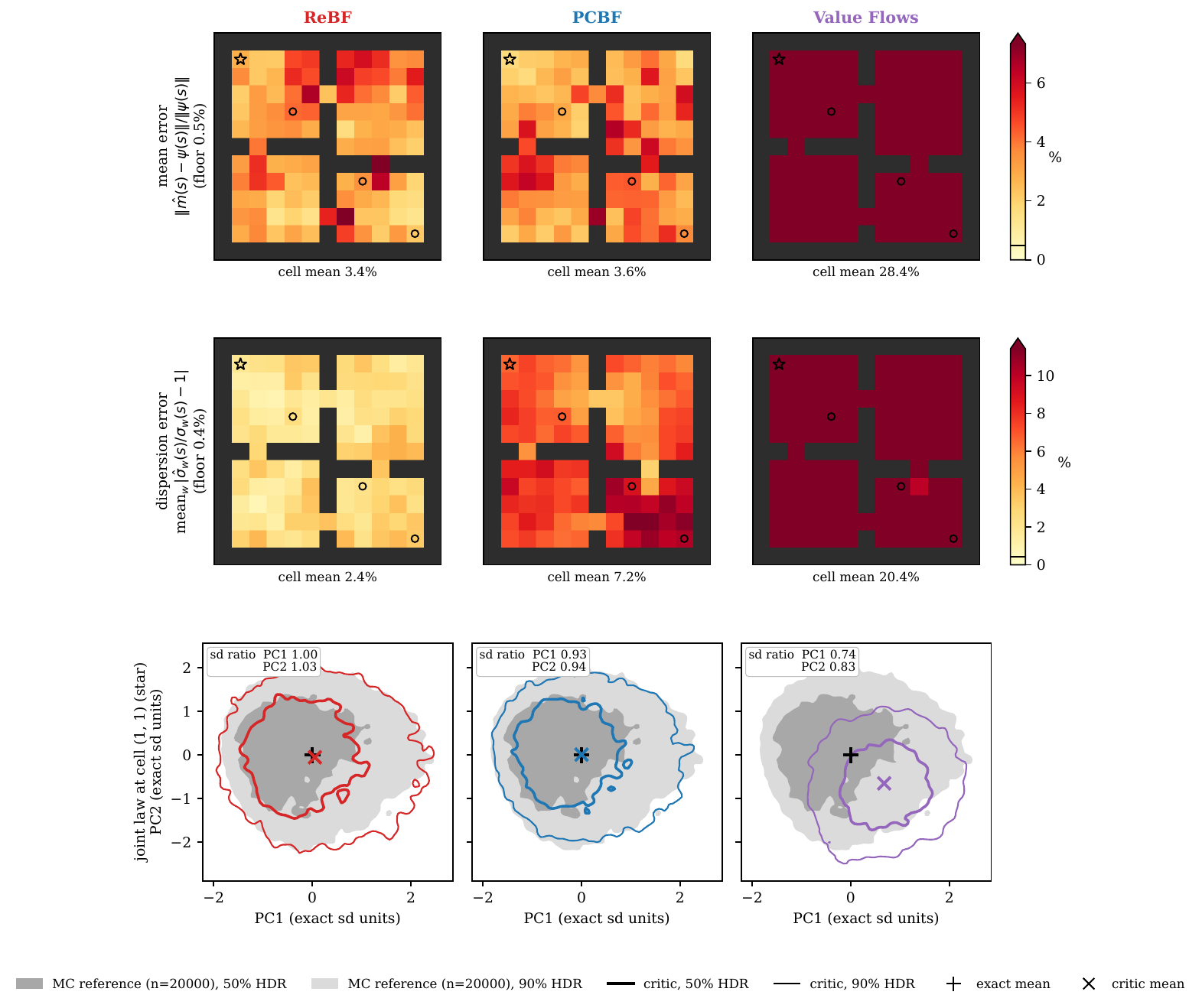}
\caption{\textbf{Four-rooms error maps}, seed $0$ ($d=16$, $\gamma=0.95$; columns ReBF, PCBF, Value Flows). Row 1: mean error $\|b\|/\|\psi\|$ over all $104$ cells. Row 2: dispersion error, dropping the ill-conditioned direction $w=M^\top e_{15}$, so the printed cell means are not the Table~\ref{tab:compare} entries. A star marks $(1,1)$ and circles the other evaluation states; the colour scale is set by ReBF and PCBF. Row 3: the joint law at $(1,1)$ on its two leading principal axes.}
\label{fig:fourrooms}
\end{figure}

\begin{figure}[t]
\centering
\includegraphics[width=\linewidth]{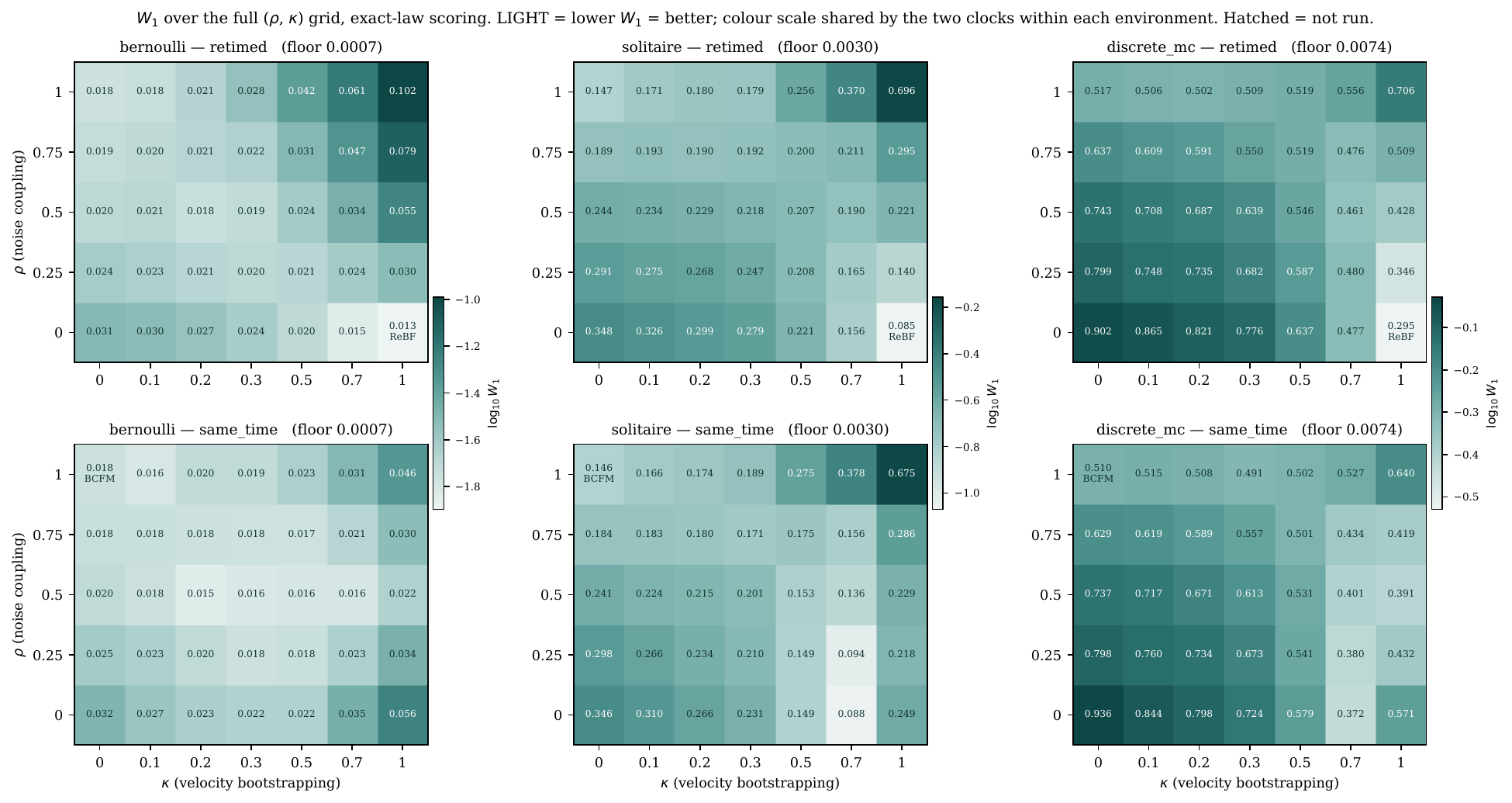}
\caption{The $(\rho_1,\kappa)$ surface: all $210$ cells, seeds $0$--$4$, $W_1$ to the exact law at $50$ ODE steps, one colour scale per environment. On these seeds the ReBF corner $(\rho_1=0,\kappa=1)$ is the lowest cell in every environment; the margin is significant only on \textsc{discrete-mc} ($p=0.025$). At $\kappa=0$ the clock is inert, and the two panels agree to within $4.2\%$ of the cell.}
\label{fig:grid}
\end{figure}

\section{Environments and Their Exact Return Laws}
\label{app:envs}
To evaluate whether a critic learns the correct return distribution, rather than inferring it from control performance, we use Markov reward processes whose return laws can be derived exactly or computed to machine precision. They expose different kinds of stochasticity in the Bellman recursion: \textsc{bernoulli} is a single-state process whose discounted return is exactly uniform; \textsc{solitaire} produces a purely atomic, long-tailed return through stochastic termination; \textsc{discrete-mc} accumulates stochasticity through repeated nearest-neighbour transitions; and a successor-feature environment, four-rooms (used in Table~\ref{tab:compare}), has vector-valued returns whose mean and directional spread are exact. For each we give the process as it is implemented, derive its return law, and state the reference used for scoring (Appendix~\ref{app:protocol}).

\subsection{\textsc{bernoulli}: one state, fair-coin rewards}
\label{app:env-bernoulli}

\paragraph{Definition.}
\textsc{bernoulli} \citep[Example~2.10]{bellemare2023} has one state and one action. Every step returns to that state and pays an independent fair coin, with $\gamma=\tfrac12$:
\begin{equation}
R_k\overset{\text{i.i.d.}}{\sim}\mathrm{Bern}(1/2),\qquad \gamma=\tfrac12 .
\label{eq:env-bernoulli-mrp}
\end{equation}
Training uses one-step backups, so the terminal multiplier $\tilde\gamma=0$ never occurs. The critic is evaluated at the only state.

\paragraph{Return law.}
\begin{equation}
G=\sum_{k\ge0}\gamma^k R_k,\qquad
\eta=\Law(G)=\mathcal U(0,2).
\label{eq:env-bernoulli-law}
\end{equation}
If $G'\sim\mathcal U(0,2)$ then $R+G'/2$ is the equal mixture of $\mathcal U(0,1)$ and $\mathcal U(1,2)$, so the uniform law is the Bellman fixed point. It is the unique fixed point because $\T^\pi$ contracts in $W_1$. The law is uniform only at $\gamma=\tfrac12$; the scoring code refuses any other discount. $\E G=1$ and $\mathrm{sd}(G)=1/\sqrt3\approx0.577$.

\paragraph{Exact reference used for scoring.}
The reference is the grid of $N+1=200{,}001$ atoms at spacing $h=10^{-5}$ on $[0,2]$, whose $W_1$ to $\mathcal U(0,2)$ is about $h/3=3.3\times10^{-6}$. The critic is read out with $4\times10^5$ draws and $50$ Euler steps. The evaluation floor, the $W_1$ of that many exact draws, is $6.5\times10^{-4}$ in Appendix~\ref{app:protocol} (about $10^{-3}$ in expectation). The coarser training-curve instrument, $4\times10^3$ draws against a $2\times10^4$-draw reference, has floor about $0.011$.

\subsection{\textsc{solitaire}}
\label{app:env-solitaire}

\paragraph{Definition.}
\textsc{solitaire} \citep[Example~2.8]{bellemare2023} has one state. Each step rolls a fair die: a $1$ pays $0$ and terminates, and $2,\dots,6$ pay $1$ and continue. With $p=1/6$ and $q=5/6$,
\begin{equation}
(R,\tilde\gamma)=(0,0)\ \text{with probability }p,\qquad (1,\gamma)\ \text{with probability }q,
\label{eq:env-solitaire-backup}
\end{equation}
and $\gamma=0.9$. The critic is evaluated at the only state.

\paragraph{Return law.}
Let $K$ be the number of non-terminating rolls, $\Pr(K=k)=p q^k$. The terminating roll pays nothing, so
\begin{equation}
G=a_K,\qquad a_k=\frac{1-\gamma^k}{1-\gamma},\qquad
\eta^\pi=\sum_{k=0}^{\infty}p q^k\,\delta_{a_k}.
\label{eq:env-solitaire-law}
\end{equation}
The atoms accumulate at $1/(1-\gamma)=10$, which carries no mass. The atom at $0$ has mass $1/6$. The mean and standard deviation are $\mu=q/(1-q\gamma)=10/3$ and $\mathrm{sd}(G)=\sqrt{800/117}\approx2.615$ ($\mathrm{CV}\approx0.784$).

\paragraph{Exact reference used for scoring.}
The reference keeps atoms $k=0,\dots,1999$, renormalised; the dropped mass is negligible. $W_1$ is the integral of the absolute difference of distribution functions against $4\times10^5$ critic draws ($50$ Euler steps). The mean and standard deviation are the exact values above. The expected evaluation floor is $4.5\times10^{-3}$; the realisation quoted in Appendix~\ref{app:protocol} is $3.0\times10^{-3}$.

\subsection{\textsc{discrete-mc}: an absorbing nearest-neighbour chain}
\label{app:env-discrete_mc}

\paragraph{Definition.}
\textsc{discrete-mc} is a birth--death chain on $\{0,1,\dots,19\}$ with absorbing ends and interior states $1,\dots,18$ \citep{cheng2024}. From an interior state the walk prefers the heavier neighbour under weights $w_i=\exp(\beta\cos(4\pi(i-1)/19))$, $\beta=(n-1)/(4\pi)$. A transition into the interior pays $1$; absorption pays $0$ and sets $\tilde\gamma=0$. Otherwise $\tilde\gamma=\gamma=0.95$.

\paragraph{Return law.}
Let $K$ be the number of rewarded steps before absorption and $Q$ the substochastic interior block. Then
\begin{equation}
G(s)=x_K,\qquad x_k=\frac{1-\gamma^k}{1-\gamma},\qquad
\Pr_s(K=k)=(Q^k u)_s,
\label{eq:env-discrete_mc-return}
\end{equation}
with $u$ the one-step absorption probabilities. Absorption is certain. Means of $G(s)$ run from $1.55$ at state $18$ to $19.23$ at state $10$; per-state coefficients of variation run from $0.101$ to $2.47$.

\paragraph{Exact reference used for scoring.}
The reference is the atomic law from the recursion above, truncated at $k=4000$, where the remaining mass is at most $4.7\times10^{-8}$. Reported $W_1$ averages the $18$ interior states, each scored with $4\times10^5$ critic draws. The mean bias uses the average of $\E[G(s)]$, which is $12.788$. The evaluation floor in Appendix~\ref{app:protocol} is $0.0074$.

\subsection{Four-rooms successor features (\texorpdfstring{$d=16$}{d=16})}
\label{app:env-fourrooms}

\paragraph{Definition.}
\textsc{four-rooms}, the vector-valued environment of Table~\ref{tab:compare} and of the fourth column of Figure~\ref{fig:allenvs}, is a lazy random walk on the four-rooms layout of \citet{sutton1999} (Figure~\ref{fig:env-fourrooms}). The layout is the $11\times11$ interior of a $13\times13$ wall map whose rows and columns are numbered $0,\dots,12$ from the top left. An interior wall runs down column $6$ and is broken by doorways at rows $3$ and $10$. The left half is split by a wall along row $6$ with a doorway at column $2$, and the right half by a wall along row $7$ with a doorway at column $9$. These $17$ interior wall cells leave $104$ open cells: a $5\times5$ north-west room (rows $1$--$5$, columns $1$--$5$), a $6\times5$ north-east room (rows $1$--$6$, columns $7$--$11$), a $5\times5$ south-west room (rows $7$--$11$, columns $1$--$5$), a $4\times5$ south-east room (rows $8$--$11$, columns $7$--$11$), and the four single-cell doorways $(3,6)$, $(10,6)$, $(6,2)$ and $(7,9)$, written (row, column). Each doorway is the only connection between its two rooms, and the rooms form a cycle, so the north-west and south-east rooms are two doorways apart either way round. The state space $\mathcal S$ is the set of open cells, indexed in row-major order. From $s$ the walk stays put with probability $p_{\rm stay}=0.2$ and otherwise moves to each of its four compass neighbours with probability $(1-p_{\rm stay})/4=0.2$, so it picks one of the five moves $\{\text{stay},\mathrm N,\mathrm S,\mathrm E,\mathrm W\}$ uniformly. A move into a wall leaves it in place:
\begin{equation}
P(s,s')=0.2\ \text{ for each open compass neighbour $s'$ of $s$},\qquad
P(s,s)=0.2+0.2\,n_{\rm wall}(s),
\label{eq:env-fourrooms-kernel}
\end{equation}
where $n_{\rm wall}(s)\in\{0,1,2\}$ counts the walls among the four neighbours of $s$, and every other entry of row $s$ is $0$. The holding probability is therefore $0.2$ at $44$ cells, $0.4$ at $40$, and $0.6$ at $20$, which are the $16$ room corners and the $4$ doorways. A move between open cells $s\to s'$ is possible exactly when $s'\to s$ is, with the same probability, so $P$ is symmetric and doubly stochastic. Its stationary law is uniform, and it has a real orthonormal eigenbasis.

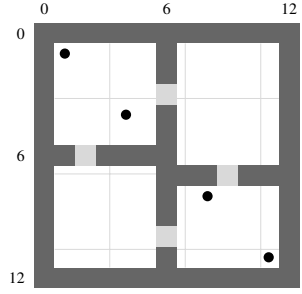
\begin{figure}[ht]
\centering
\begin{tikzpicture}[x=0.27cm,y=-0.27cm]
  \fill[black!60] (0,0) rectangle (13,13);
  \fill[white] (1,1) rectangle (12,12);
  \draw[black!15,very thin] (1,1) grid (12,12);
  \foreach \c/\r in {6/1,6/2,6/4,6/5,1/6,3/6,4/6,5/6,6/6,6/7,7/7,8/7,10/7,11/7,6/8,6/9,6/11}
    \fill[black!60] (\c,\r) rectangle ++(1,1);
  \foreach \c/\r in {6/3,2/6,9/7,6/10}
    \fill[black!15] (\c,\r) rectangle ++(1,1);
  \foreach \c/\r in {1/1,4/4,8/8,11/11}
    \fill (\c+0.5,\r+0.5) circle (0.25);
  \foreach \k in {0,6,12} {
    \node[font=\tiny,above] at (\k+0.5,0) {\k};
    \node[font=\tiny,left] at (0,\k+0.5) {\k};
  }
\end{tikzpicture}
\caption{\textbf{The \textsc{four-rooms} layout} (Appendix~\ref{app:env-fourrooms}). Walls are dark, the four doorways light grey and the other open cells white, $104$ open cells in all. Dots mark the four evaluation states $(1,1)$, $(4,4)$, $(8,8)$ and $(11,11)$ (row, column; rows numbered downward).}
\label{fig:env-fourrooms}
\end{figure}

\paragraph{Features.}
$P$ is symmetric, so it has a real orthonormal eigenbasis, computed once and stored. We keep $d=16$ modes, log-spaced from the slowest non-constant eigenvalue $\lambda=0.9954$ to the bottom of the spectrum, excluding the constant mode. Let $V$ hold those eigenvectors and let $M$ be one fixed random orthogonal matrix, shared by every arm and seed. Then
\begin{equation}
B=VM,\qquad \phi(s)=B_{s,:}^{\top}\in\R^{16},\qquad R=\phi(S),\qquad \tilde\gamma=\gamma=0.95 .
\label{eq:env-fourrooms-features}
\end{equation}
It follows that $B^\top B=I_{16}$, that every feature averages to zero over the $104$ cells, and that $0.20\le\norm{\phi(s)}\le0.66$. The reward is paid in the current state. There is no termination, so the backup multiplier of Eq.~\eqref{eq:bellman} equals $\gamma$ on every transition. Training draws the current state uniformly from all $104$ cells and takes one step of $P$.

\paragraph{Return law.}
The return is the discounted successor-feature vector
\begin{equation}
\Psi(s)=\sum_{k\ge0}\gamma^k\phi(S_k),\quad S_0=s,\ S_{k+1}\sim P(S_k,\cdot);
\qquad
\Psi(s)\overset{d}{=}\phi(s)+\gamma\,\Psi(S_1).
\label{eq:env-fourrooms-return}
\end{equation}
By the Markov property, $\Psi(S_1)$ given $S_1$ is an independent draw from $\eta_{S_1}$. So $\eta_s=\Law(\Psi(s))$ is the unique fixed point of $\eta_s=\sum_{s'}P(s,s')\,(T_s)_\#\eta_{s'}$ with $T_s(z)=\phi(s)+\gamma z$. It is a mixture of at most five affine images of the successors' laws (only three at the room corners and doorways), supported in the ball of radius $\max_s\norm{\phi(s)}/(1-\gamma)=13.1$. On the three scalar environments the return is a function of a single stopping time, or a binary expansion of independent coins. Here it depends on the whole path through $\phi$, and unrolling the recursion sums over infinitely many paths. We have no closed form for $\eta_s$, nor for the law of any of its one-dimensional slices, but the first two moments of every slice are exact. For $w\in\R^{16}$, $G_w(s)=w^\top\Psi(s)$ is the return of the scalar process on the same chain with reward $r_w=Bw$. The critic is fitted to the full vector law and never sees $w$, so every slice is a zero-shot readout.

\paragraph{Moments.}
The mean $\psi(s)=\E\Psi(s)$ and the variance of each slice $G_w=w^\top\Psi$ are exact:
\begin{equation}
\psi=(I-\gamma P)^{-1}B,
\label{eq:env-fourrooms-mean}
\end{equation}
and, with $r_w=Bw$ and $m_w=\psi w$,
\begin{equation}
\sigma_w^2=\gamma^2(I-\gamma^2 P)^{-1}\big[P(m_w^{\,2})-(Pm_w)^2\big].
\label{eq:env-fourrooms-second}
\end{equation}

\paragraph{Exact reference used for scoring.}
The mean and the standard deviations are exact. The mean error is $\norm{b(s)}/\norm{\psi(s)}$ averaged over the four evaluation states below, where $b(s)$ is the critic's sample mean minus $\psi(s)$ from Eq.~\eqref{eq:env-fourrooms-mean}. The dispersion error is $\sum|\widehat{\rm sd}-\sigma_w(s)|/\sum\sigma_w(s)$ over the $44$ slices below, with $\sigma_w$ from Eq.~\eqref{eq:env-fourrooms-second}. $W_1$ has no exact reference here, so unlike on the scalar environments it is scored against a \emph{sampled} one. From each evaluation state we draw $N=4\times10^5$ independent rollouts of $P$ and truncate each at $T=400$ steps, which moves a draw by at most $\gamma^{T}\max_s\norm{\phi(s)}/(1-\gamma)<1.7\times10^{-8}$. As a check on the sampled references, across the fifteen per-seed references the largest of the $960$ per-coordinate deviations of the sample mean from $\psi(s)$ is $3.1$ standard errors, and every one of the $660$ per-slice sample standard deviations lies within $1.2\%$ of its exact $\sigma_w$. The evaluation states are the open cells at positions $0$, $34$, $69$ and $103$ of the row-major order. They are the corner $(1,1)$ and the interior cell $(4,4)$ of the north-west room, and the cell $(8,8)$ beside the north wall and the corner $(11,11)$ of the south-east room. No cell of the north-east or south-west room, and no doorway, is scored. At each evaluation state, $2\times10^4$ critic draws ($128$ Heun steps) and the reference are projected onto $11$ fixed unit directions:
\begin{itemize}
\item four eigen-directions $w=M^\top e_k$, $k\in\{0,5,10,15\}$. For these $r_w$ is a single eigenvector with $\lambda=0.9954$, $0.9263$, $0.6531$ and $-0.4774$ (horizons $18.4$, $8.3$, $2.6$ and $0.69$), so $m_w=r_w/(1-\gamma\lambda)$. The first is the slowest room mode;
\item three feature directions $w=e_j$, $j\in\{0,8,15\}$, each a fixed mixture of the $16$ kept modes;
\item four random directions $g/\norm g$ with $g\sim\mathcal N(0,I_{16})$, drawn once.
\end{itemize}
With $n=2\times10^4$ midpoint quantile levels $u_i=(i-\frac12)/n$, the score is
\begin{equation}
W_1^{\rm sliced}=\frac1{44}\sum_{s}\sum_{w}\frac1n\sum_{i=1}^{n}\Big|\hat F^{-1}_{w^\top\hat X_s}(u_i)-\hat F^{-1}_{w^\top\hat\Psi_s}(u_i)\Big|,
\label{eq:env-fourrooms-w1}
\end{equation}
where $\hat X_s$ are the critic's draws at $s$ and $\hat\Psi_s$ the reference draws. $\hat F^{-1}$ is the empirical quantile function, with linear interpolation. The map $z\mapsto w^\top z$ is $1$-Lipschitz for a unit $w$, so each slice lower-bounds the $16$-dimensional $W_1$. An error orthogonal to all $11$ directions is invisible to the score, which is why the exact mean and per-slice spread are reported beside it. Each training seed has its own reference draw, shared by every arm at that seed. The evaluation floor is the same score between a fresh $2\times10^4$-draw sample of the law and the reference, replicate $j$ using the reference of training seed $j$. Over five replicates it is $0.00460\pm0.00026$ (mean $\pm$ s.e.), about $0.9\%$ of the median $\sigma_w$. At $20{,}000$ steps the fifteen-seed ReBF, PCBF and BCFM runs score $0.0282\pm0.0009$, $0.0410\pm0.0010$ and $0.0523\pm0.0007$ (mean $\pm$ s.e.), which is $6.1$, $8.9$ and $11.4$ times this floor.

\section{Experimental Protocol}
\label{app:protocol}

\paragraph{Environments.} Three finite Markov reward processes with analytically known return laws, evaluated under a fixed policy: \textsc{bernoulli} ($\gamma=0.5$), \textsc{solitaire} ($\gamma=0.9$, a stochastic termination roll) and \textsc{discrete-mc} ($18$ non-terminal states, $\gamma=0.95$). Their return laws, derived in Appendix~\ref{app:envs}, differ sharply in dispersion (Section~\ref{sec:toys}).

\paragraph{Arms.} Every arm is a triple (query clock, coupling $\rho_1$, correction weight $\kappa$) and is named by that triple; a single source file defines them, so no arm is specified twice. At $\kappa=0$ the query clock is arithmetically inert, and we verified that the first training target is bit-identical under both clock flags; the two settings are one arm, run once for the headline comparisons; the two clock flags were trained separately as a check, and those twin runs differ by up to $4\%$ of level.

\paragraph{Training and evaluation.} $50$k steps, batch $512$, shared architecture and optimiser, $20$-step Euler integration during training. Evaluation uses $50$ ODE steps and $4\times10^5$ draws against the exact law; for \textsc{discrete-mc} the reported $W_1$ averages all $18$ non-terminal states. Reference-error floors under a perfect model are $6.5\times10^{-4}$, $3.0\times10^{-3}$ and $7.4\times10^{-3}$ respectively, well below every reported difference.

\paragraph{Selection and pairing.} Tuned arms select $\kappa$ on five seeds and are scored on five disjoint seeds; both split directions are reported. Comparisons are paired by seed, and ties are declared using the $t$ critical value at the relevant sample size. Seeds share random-number streams across arms but not bit-identical trajectories: two runs with the same seed and an identical first target can differ by a few tenths of a percent after thousands of steps through floating-point accumulation alone, which bounds the resolution of any paired comparison at that scale.

\paragraph{The evaluation floor, and what a power control certifies.}
Every metric we report compares a finite sample from the learned critic against a finite reference, so it has a positive floor even for an exactly correct critic. We measure that floor rather than assume it: we draw twice from the \emph{true} law at the sample sizes the evaluation uses and compute the same metric. For the depth sweep of Table~\ref{tab:depth} ($4{,}000$ learned draws against a $20{,}000$-draw reference) this gives a $W_1$ floor of $0.0432\pm0.0166$ on \textsc{solitaire} at $\gamma=0.9$. Separations are then reported as multiples of the relevant floor.

The main results are scored on a far finer instrument than the depth sweep: they use $4\times10^5$ learned draws against an \emph{exact} closed-form reference, which removes reference-side noise entirely, giving floors of $6.5\times10^{-4}$, $3.0\times10^{-3}$ and $7.4\times10^{-3}$ on the three environments (each a single draw of a random quantity: the expected values are $1.0\times10^{-3}$ and $4.5\times10^{-3}$ on \textsc{bernoulli} and \textsc{solitaire}, and the \textsc{discrete-mc} value averages states $1$--$4$, its $18$-state value being $0.006$--$0.009$ across draws; Appendix~\ref{app:envs}) --- one to two orders of magnitude below the depth sweep's floor. Every headline level is at least $19\times$ its floor, and the headline gaps on \textsc{solitaire} and \textsc{discrete-mc} are significant in both split directions ($p\le0.04$ at $n=5$); the \textsc{bernoulli} gaps are $1.3$--$2.2\times$ the floor at $n=15$, and we report them as a tie. The main comparisons were therefore never near the resolution limit where they claim a winner; it is the depth sweep, run at $4{,}000$ draws against a sampled reference, where the floor became load-bearing.

Two consequences are applied throughout. First, a claim that two arms \emph{converge} requires the instrument to resolve the difference it asserts has vanished; where a gap falls below the floor's own standard deviation we say the arms are indistinguishable at the setting tested, and do not say that one caught up. Second, a power control certifies only the arm it is placed on. A control confirming that the \emph{baseline} arm responds to a manipulation establishes that the cell resolves baseline-sized effects at the baseline's position in the instrument's range; it establishes nothing about an arm whose error sits an order of magnitude lower. We therefore measure the floor for the arm that \emph{carries} the claim, at the sample size actually used.

We adopted this after discarding a $75$-run sweep in which a preregistered power control passed --- the baseline arm fell by $1.49\times$, paired $p=0.022$ --- while the arm whose profile was the readout sat at $0.9$ to $1.15$ times the measured floor. An exactly correct critic would have produced the same curve. The sweep completed, the registered rule passed, and the magnitudes, shape and error bars all looked ordinary: the failure had no symptom. Preregistration, held-out selection, paired seeds and convergence checking do not detect a power analysis aimed at the wrong arm; measuring the floor for the right one does, which is why the check is unconditional here rather than applied on suspicion.

\paragraph{Analytic references.} Return-law references are derived from the environment dynamics rather than from any helper shipped with an environment; one such helper contained an off-by-one in the discount exponent, and results computed against it were discarded.

\section{Algorithm}
\label{app:algo}

\begin{algorithm}[h]
\small
\caption{One update of the retimed family; ReBF is $\rho_1=0$, $\kappa=1$}
\label{alg:sw}
\begin{enumerate}\itemsep1pt\parskip0pt
\item Sample backup $H=(R,S',A',\tilde\gamma)$ for $c$; independently draw $X_0,E,t$.
\item If $\tilde\gamma=0$, set $X_t=(1-t)X_0+tR$, $u=R-X_0$, and skip to step 7.
\item Set $X_0'=\rho_1 X_0+\sqrt{1-\rho_1^2}E$ and $X_1'=\psib^{\,1}(X_0'\mid S',A')$.
\item Set $X_t=(1-t)X_0+t(R+\tilde\gamma X_1')$ and $C=X_1'-X_0$.
\item Compute $\alpha,\tau,g$ by Eqs.~\eqref{eq:clock}--\eqref{eq:identity}; set $\xi=(X_t-tR)/\alpha$.
\item Query the frozen teacher and set $u=R+\tilde\gamma X_1'-X_0+\kappa g(\bar v(\tau,\xi)-C)$.
\item Regress $v_\theta(t,X_t\mid c)$ onto $\sg(u)$; average target parameters.
\end{enumerate}
\end{algorithm}


\begin{table}[ht]
\centering
\caption{Terminal scale error of the same-time query with constant gain $\tilde\gamma$ under an exact teacher at $\rho_1=0$, $\gamma=0.99$, as a percentage of the correct successor scale. Gaussian: Eq.~\eqref{eq:scalefree} in closed form, identical for every $\sigma$. Bernoulli: the same constant-gain target on a $256$-atom quantile compression of the \textsc{bernoulli} return law, estimated by fixed-seed binned Monte Carlo ($2\times10^7$ draws, seed $0$, $100$ time bins, $200$ state bins, $100$ Euler steps, $2\times10^5$ evaluation draws). Retimed: the same estimator applied to the retimed target on the Bernoulli law. Its population value is exactly zero (Theorem~\ref{thm:centering} and Proposition~\ref{prop:scale}), so this column measures the estimator's floor, which also enters the Bernoulli column. None of these are averages over training seeds.}
\label{tab:regime}
\begin{tabular}{rrrrr}
\toprule
Backup $n$&$\tilde\gamma=0.99^n$&Gaussian (any $\sigma$)&Bernoulli law&Retimed (floor)\\
\midrule
1   & 0.9900 & \phantom{0}0.50\% & \phantom{0}0.45\% & $+0.09\%$\\
5   & 0.9510 & \phantom{0}2.48\% & \phantom{0}1.84\% & $+0.10\%$\\
10  & 0.9044 & \phantom{0}4.89\% & \phantom{0}3.50\% & $+0.10\%$\\
25  & 0.7778 & 11.62\% & \phantom{0}7.86\% & $+0.13\%$\\
50  & 0.6050 & 20.86\% & 13.22\% & $+0.19\%$\\
\bottomrule
\end{tabular}
\end{table}

\section{Online supplement: inference budget and per-task OGBench}
The integration-budget study and the $38$-task OGBench table are supplementary to the main results.

\subsection{Inference-time integration budget}
\label{app:budget}
Figure~\ref{fig:budget} reads the same trained critics out with $1$ to $200$ Euler steps, so it measures the readout and not the critic. From $10$ to $200$ steps ReBF improves by a factor of two to four on every environment, and pinned PCBF by $1.3$--$2.1\times$. The same-time independent-noise arm, which the figure does not draw, gets \emph{worse} on two of three ($0.030\to0.062$ on \textsc{bernoulli} and $0.242\to0.285$ on \textsc{solitaire}, paired $p\le0.023$; $0.662\to0.624$ on \textsc{discrete-mc}, $p=0.29$; seeds $0$--$4$): a biased field is not rescued by integrating it more accurately, and a coarse readout hides the bias. On \textsc{bernoulli} and \textsc{solitaire} tuned PCBF nearly coincides with the uncorrected critic, because held-out tuning (on seeds $5$--$9$) selects a small $\kappa^\star$ ($0.02$ and $0.05$; $0.3$ on \textsc{discrete-mc}). On \textsc{discrete-mc} tuned PCBF is a single point at $50$ steps: the $\kappa$-ladder arms were re-scored on a shorter budget grid, and each arm is drawn only where it has data.

\begin{figure}[H]
\centering
\includegraphics[width=\linewidth]{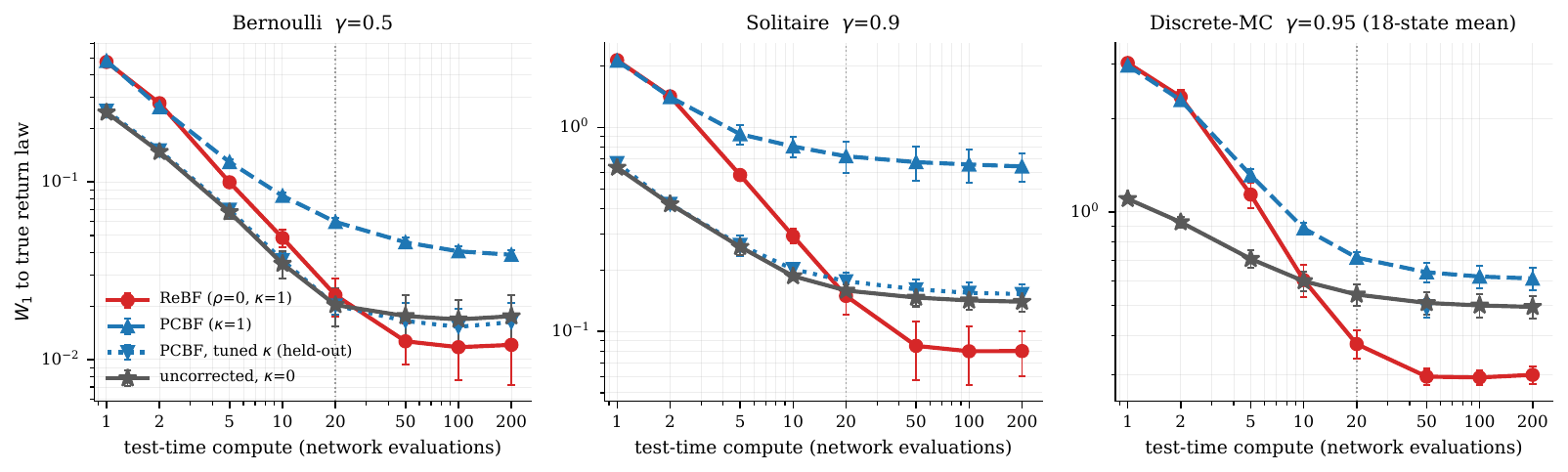}
\caption{Test-time compute: $W_1$ to the exact law ($4\times10^5$ draws) as the same trained critics are read out with $1$ to $200$ Euler steps (dotted line: the training resolution, $20$), mean $\pm1$ s.d.\ over five paired seeds; tuned PCBF uses the $\kappa^\star$ selected on disjoint seeds. Lower is better; the $x$-axis is the readout budget, not training.}
\label{fig:budget}
\end{figure}

\subsection{Per-task OGBench results}
\label{app:fullogbench}
\begin{table}[H]
    \centering
    \caption{\textbf{Full Offline RL Results.} Performance on 38 OGBench \citep{park2025} and D4RL \citep{fu2020} tasks. The symbol $(*)$ indicates the task used for hyperparameter tuning within each domain. Results for IQN \citep{dabney2018iqn}, CODAC \citep{ma2021}, FloQ \citep{agrawalla2026}, FQL \citep{park2025fql}, IQL \citep{kostrikov2022} and Value Flows \citep{dong2026} are from \citet[Table~3]{xu2026}; PCBF results are from the supplied PCBF comparison table. ReBF results use the best evaluation checkpoint of each seed. \textbf{Bold numbers denote mean scores that are at least $95\%$ of the highest reported mean on each task.}}
    \label{tab:fullogbench_comparison}
    \resizebox{\textwidth}{!}{%
    \begin{tabular}{lcccccccc}
        \toprule
        & \multicolumn{8}{c}{\textbf{Algorithms}} \\
        \cmidrule(lr){2-9}
        \textbf{Task} & IQN & CODAC & FloQ & FQL & IQL
        & Value Flows & PCBF & \textbf{ReBF (Ours)} \\
        \midrule
        cube-double-play-singletask-task1-v0
        & $70\pm14$ & $80\pm11$ & $50\pm24$ & $61\pm9$ & $27\pm5$
        & $\mathbf{97\pm1}$ & $92\pm5$ & $\mathbf{97\pm3}$ \\
        cube-double-play-singletask-task2-v0 ($*$)
        & $24\pm9$ & $63\pm4$ & $72\pm15$ & $36\pm6$ & $1\pm1$
        & $76\pm7$ & $\mathbf{84\pm8}$ & $\mathbf{82\pm9}$ \\
        cube-double-play-singletask-task3-v0
        & $25\pm6$ & $66\pm9$ & $57\pm14$ & $22\pm5$ & $0\pm0$
        & $73\pm4$ & $80\pm7$ & $\mathbf{93\pm2}$ \\
        cube-double-play-singletask-task4-v0
        & $10\pm1$ & $13\pm2$ & $8\pm4$ & $5\pm2$ & $0\pm0$
        & $\mathbf{30\pm5}$ & $\mathbf{29\pm5}$ & $25\pm6$ \\
        cube-double-play-singletask-task5-v0
        & $\mathbf{81\pm8}$ & $\mathbf{82\pm4}$ & $50\pm11$ & $19\pm10$ & $4\pm3$
        & $69\pm5$ & $65\pm8$ & $\mathbf{83\pm8}$ \\
        \midrule
        scene-play-singletask-task1-v0
        & $\mathbf{100\pm0}$ & $\mathbf{99\pm0}$ & $\mathbf{100\pm1}$
        & $\mathbf{100\pm0}$ & $94\pm3$ & $\mathbf{99\pm0}$
        & $\mathbf{100\pm0}$ & $\mathbf{100\pm0}$ \\
        scene-play-singletask-task2-v0 ($*$)
        & $1\pm0$ & $85\pm4$ & $83\pm10$ & $76\pm9$ & $12\pm3$
        & $\mathbf{97\pm1}$ & $86\pm11$ & $\mathbf{94\pm9}$ \\
        scene-play-singletask-task3-v0
        & $94\pm2$ & $90\pm3$ & $\mathbf{98\pm2}$ & $\mathbf{98\pm1}$
        & $32\pm7$ & $94\pm2$ & $\mathbf{99\pm1}$ & $\mathbf{99\pm1}$ \\
        scene-play-singletask-task4-v0
        & $3\pm1$ & $0\pm0$ & $\mathbf{9\pm7}$ & $5\pm1$ & $0\pm1$
        & $7\pm17$ & $3\pm3$ & $\mathbf{9\pm11}$ \\
        scene-play-singletask-task5-v0
        & $\mathbf{0\pm0}$ & $\mathbf{0\pm0}$ & $\mathbf{0\pm0}$
        & $\mathbf{0\pm0}$ & $\mathbf{0\pm0}$ & $\mathbf{0\pm0}$
        & $\mathbf{0\pm0}$ & $\mathbf{0\pm0}$ \\
        \midrule
        puzzle-4x4-play-singletask-task1-v0
        & $41\pm2$ & $37\pm32$ & $\mathbf{47\pm7}$ & $34\pm8$ & $12\pm2$
        & $36\pm3$ & $39\pm6$ & $\mathbf{46\pm13}$ \\
        puzzle-4x4-play-singletask-task2-v0
        & $12\pm4$ & $10\pm10$ & $21\pm6$ & $16\pm5$ & $7\pm4$
        & $27\pm5$ & $\mathbf{28\pm4}$ & $\mathbf{29\pm3}$ \\
        puzzle-4x4-play-singletask-task3-v0
        & $\mathbf{45\pm7}$ & $33\pm29$ & $36\pm5$ & $18\pm5$ & $9\pm3$
        & $30\pm4$ & $38\pm6$ & $\mathbf{45\pm10}$ \\
        puzzle-4x4-play-singletask-task4-v0 ($*$)
        & $23\pm2$ & $12\pm10$ & $19\pm5$ & $11\pm3$ & $5\pm2$
        & $28\pm5$ & $35\pm5$ & $\mathbf{43\pm8}$ \\
        puzzle-4x4-play-singletask-task5-v0
        & $16\pm6$ & $10\pm8$ & $16\pm7$ & $7\pm3$ & $4\pm1$
        & $13\pm2$ & $\mathbf{18\pm6}$ & $15\pm4$ \\
        \midrule
        cube-triple-play-singletask-task1-v0 ($*$)
        & $29\pm2$ & $9\pm5$ & $32\pm13$ & $20\pm6$ & $4\pm4$
        & $\mathbf{59\pm12}$ & $25\pm9$ & $55\pm15$ \\
        cube-triple-play-singletask-task2-v0
        & $0\pm0$ & $\mathbf{1\pm0}$ & $0\pm0$ & $\mathbf{1\pm2}$
        & $0\pm0$ & $0\pm0$ & $0\pm1$ & $0\pm0$ \\
        cube-triple-play-singletask-task3-v0
        & $1\pm0$ & $0\pm0$ & $\mathbf{7\pm3}$ & $0\pm0$ & $0\pm0$
        & $\mathbf{7\pm3}$ & $3\pm1$ & $6\pm3$ \\
        cube-triple-play-singletask-task4-v0
        & $\mathbf{0\pm0}$ & $\mathbf{0\pm0}$ & $\mathbf{0\pm0}$
        & $\mathbf{0\pm0}$ & $\mathbf{0\pm0}$ & $\mathbf{0\pm0}$
        & $\mathbf{0\pm1}$ & $\mathbf{0\pm0}$ \\
        cube-triple-play-singletask-task5-v0
        & $0\pm0$ & $0\pm0$ & $0\pm0$ & $0\pm0$ & $1\pm1$
        & $\mathbf{2\pm1}$ & $0\pm1$ & $0\pm0$ \\
        \midrule
        pen-cloned-v1
        & $\mathbf{80\pm11}$ & $76\pm2$ & $72\pm5$ & $74\pm11$
        & $\mathbf{83}$ & $73\pm5$ & $78\pm5$ & $77\pm3$ \\
        pen-expert-v1
        & $118\pm19$ & $\mathbf{136\pm2}$ & $\mathbf{140\pm8}$
        & $\mathbf{142\pm6}$ & $128$ & $117\pm3$
        & $\mathbf{137\pm4}$ & $129\pm5$ \\
        door-cloned-v1
        & $0\pm0$ & $0\pm0$ & $\mathbf{3\pm2}$ & $2\pm1$
        & $\mathbf{3}$ & $0\pm0$ & $1\pm0$ & $1\pm0$ \\
        door-expert-v1
        & $\mathbf{105\pm0}$ & $\mathbf{104\pm0}$ & $\mathbf{104\pm0}$
        & $\mathbf{104\pm1}$ & $\mathbf{107}$ & $\mathbf{104\pm1}$
        & $\mathbf{106\pm0}$ & $\mathbf{106\pm0}$ \\
        hammer-cloned-v1
        & $0\pm0$ & $6\pm0$ & $10\pm8$ & $\mathbf{11\pm9}$
        & $2$ & $1\pm0$ & $2\pm1$ & $2\pm0$ \\
        hammer-expert-v1
        & $121\pm7$ & $\mathbf{126\pm1}$ & $\mathbf{125\pm2}$
        & $\mathbf{125\pm3}$ & $\mathbf{129}$ & $\mathbf{125\pm5}$
        & $\mathbf{126\pm2}$ & $\mathbf{126\pm2}$ \\
        relocate-cloned-v1
        & $0\pm0$ & $0\pm0$ & $0\pm0$ & $0\pm0$
        & $\mathbf{2}$ & $0\pm0$ & $0\pm0$ & $0\pm0$ \\
        relocate-expert-v1
        & $103\pm0$ & $103\pm2$ & $\mathbf{108\pm2}$
        & $\mathbf{107\pm1}$ & $\mathbf{106}$ & $102\pm2$
        & $\mathbf{109\pm1}$ & $\mathbf{108\pm1}$ \\
        \midrule
        visual-antmaze-teleport-navigate-singletask-task1-v0 ($*$)
        & $2\pm1$ & $-$ & $-$ & $2\pm1$ & $5\pm2$
        & $\mathbf{10\pm4}$ & $8\pm2$ & $9\pm3$ \\
        visual-antmaze-teleport-navigate-singletask-task2-v0
        & $7\pm3$ & $-$ & $-$ & $6\pm1$ & $10\pm2$
        & $17\pm5$ & $\mathbf{19\pm4}$ & $\mathbf{19\pm4}$ \\
        visual-antmaze-teleport-navigate-singletask-task3-v0
        & $6\pm4$ & $-$ & $-$ & $9\pm4$ & $7\pm7$
        & $16\pm3$ & $\mathbf{18\pm4}$ & $16\pm2$ \\
        visual-antmaze-teleport-navigate-singletask-task4-v0
        & $4\pm2$ & $-$ & $-$ & $9\pm1$ & $4\pm6$
        & $16\pm5$ & $\mathbf{18\pm5}$ & $15\pm4$ \\
        visual-antmaze-teleport-navigate-singletask-task5-v0
        & $2\pm1$ & $-$ & $-$ & $1\pm1$ & $2\pm1$
        & $\mathbf{8\pm2}$ & $\mathbf{8\pm3}$ & $7\pm2$ \\
        \midrule
        visual-cube-double-play-singletask-task1-v0 ($*$)
        & $4\pm1$ & $-$ & $-$ & $23\pm4$ & $\mathbf{34\pm23}$
        & $\mathbf{35\pm2}$ & $10\pm3$ & $11\pm4$ \\
        visual-cube-double-play-singletask-task2-v0
        & $0\pm0$ & $-$ & $-$ & $0\pm0$ & $3\pm1$
        & $\mathbf{4\pm2}$ & $0\pm0$ & $1\pm1$ \\
        visual-cube-double-play-singletask-task3-v0
        & $0\pm0$ & $-$ & $-$ & $0\pm0$ & $7\pm4$
        & $\mathbf{11\pm2}$ & $0\pm0$ & $1\pm1$ \\
        visual-cube-double-play-singletask-task4-v0
        & $0\pm0$ & $-$ & $-$ & $\mathbf{4\pm2}$ & $2\pm1$
        & $2\pm1$ & $0\pm0$ & $0\pm0$ \\
        visual-cube-double-play-singletask-task5-v0
        & $1\pm1$ & $-$ & $-$ & $4\pm1$ & $11\pm2$
        & $\mathbf{13\pm3}$ & $3\pm1$ & $3\pm1$ \\
        \bottomrule
    \end{tabular}%
    }
\end{table}

\end{document}